\documentclass[letterpaper]{article} 
\usepackage{aaai25}  
\usepackage{times}  
\usepackage{helvet}  
\usepackage{courier}  
\usepackage[hyphens]{url}  
\usepackage{graphicx} 
\usepackage{natbib}  
\usepackage{caption} 
\usepackage{algorithm}
\usepackage{algorithmic}

\usepackage{newfloat}
\usepackage{listings}
\DeclareCaptionStyle{ruled}{labelfont=normalfont,labelsep=colon,strut=off} 
\floatstyle{ruled}
\newfloat{listing}{tb}{lst}{}
\floatname{listing}{Listing}
\usepackage{amsmath,amsfonts,bm}

\def\eqref#1{equation~\ref{#1}}

\def\1{\bm{1}}

\DeclareMathAlphabet{\mathsfit}{\encodingdefault}{\sfdefault}{m}{sl}
\SetMathAlphabet{\mathsfit}{bold}{\encodingdefault}{\sfdefault}{bx}{n}

\newcommand{\R}{\mathbb{R}}

\newcommand{\softmax}{\mathrm{softmax}}

\DeclareMathOperator{\sign}{sign}

\usepackage{pgfplots}

\usepackage{dsfont}

\newcommand\N{{\mathbb N}}
\newcommand\nan{{\texttt{NaN}}}
\usepackage{arydshln}
\usepackage[notrig]{physics}
\usepackage{mathtools}
\usepackage{subcaption}
\usepackage{caption}
\makeatletter
\newcommand\footnoteref[1]{\protected@xdef\@thefnmark{\ref{#1}}\@footnotemark}
\makeatother
\DeclareMathAlphabet\mathbfcal{OMS}{cmsy}{b}{n}
\usepackage{scalerel}
\usepackage{graphicx}
\usepackage{booktabs}
\usepackage{amsmath}
\usepackage{amsthm}
\usepackage{amsfonts}
\usepackage{epsfig}
\usepackage{graphicx}
\usepackage[bottom]{footmisc}
\usepackage{algorithm}
\usepackage{comment}
\usepackage{xspace}
\usepackage[singlelinecheck=off]{caption}
\DeclareCaptionType{copyrightbox}
\usepackage{pifont}
\newtheorem{thrm}{Theorem}

\newtheorem{lma}{Lemma}

\newtheorem{exple}{Example}
 
\newcommand{\squishlist}{
 \begin{list}{$\bullet$}
  {  \setlength{\itemsep}{0pt}
     \setlength{\parsep}{3pt}
     \setlength{\topsep}{3pt}
     \setlength{\partopsep}{0pt}
     \setlength{\leftmargin}{2em}
     \setlength{\labelwidth}{1.5em}
     \setlength{\labelsep}{0.5em}
} }
\newcommand{\squishlisttight}{
 \begin{list}{$\bullet$}
  { \setlength{\itemsep}{0pt}
    \setlength{\parsep}{0pt}
    \setlength{\topsep}{0pt}
    \setlength{\partopsep}{0pt}
    \setlength{\leftmargin}{2em}
    \setlength{\labelwidth}{1.5em}
    \setlength{\labelsep}{0.5em}
} }

\newcommand{\squishdesc}{
 \begin{list}{}
  {  \setlength{\itemsep}{0pt}
     \setlength{\parsep}{3pt}
     \setlength{\topsep}{3pt}
     \setlength{\partopsep}{0pt}
     \setlength{\leftmargin}{1em}
     \setlength{\labelwidth}{1.5em}
     \setlength{\labelsep}{0.5em}
} }

\newcommand{\squishend}{
  \end{list}
}

\newcommand{\eat}[1]{}

\newcounter{ccc}

\usepackage{tikz}
\usepackage{pgfplots}
\usetikzlibrary{arrows.meta}
\pgfplotsset{compat = newest}
\usepackage{xcolor, soul}
\usetikzlibrary{shapes}
\usetikzlibrary{calc}
\usepackage{MnSymbol}

\newcommand{\bigO}{\mathcal{O}}

\usepackage{url}

\newcommand{\act}{\textnormal{Shiesh}}

\newcommand{\fc}{\text{SITA}}

\newcommand{\ft}{\textnormal{EL}}

\newcommand{\fa}{\text{Shiesh}}

\newcommand{\bftab}[1]{{\fontseries{b}\selectfont #1}}
\newcommand{\xobs}{x^\textsc{obs}}
\newcommand{\xq}{x^\textsc{qry}}
\newcommand{\qu}{^\textsc{qry}}
\newcommand{\qem}{\text{GraFITi}}
\newcommand{\scale}{\text{NN}\textsuperscript{sca}}
\newcommand{\transl}{\text{NN}\textsuperscript{trs}}
\newcommand{\obs}{^\textsc{obs}}
\newcommand{\D}{\mathcal{D}}
\newcommand\normal{{\cal N}}
\newcommand\commentout[1]{}
\DeclareMathOperator\argsort{\text{argsort}}
\newcommand\id{\mathbb{I}}
\newcommand{\com}{^\text{com}}

\newcommand{\satt}{\text{Attn}}
\newcommand{\isatt}{\text{iAttn}}
\newcommand{\stisatt}{\text{SITA}}

\title{Probabilistic Forecasting of Irregularly Sampled Time Series with Missing Values via Conditional Normalizing Flows}

\author {
	Vijaya Krishna Yalavarthi\equalcontrib \textsuperscript{\rm 1}, 
	Randolf Scholz\equalcontrib \textsuperscript{\rm 1}, 
	Stefan Born \textsuperscript{\rm 2}, 
	Lars Schmidt-Thieme \textsuperscript{\rm 1}
}
\affiliations {
	\textsuperscript{\rm 1} Information Systems and Machine Learning Lab, University of Hildesheim, Germany\\
	\textsuperscript{\rm 2} Institute of Mathematics, TU Berlin, Germany\\ 
	\{yalavarthi, scholz, schmidt-thieme\}@ismll.de,\\ born@math.tu-berlin.de
}

\begin{document}

\maketitle

\begin{abstract}
	Probabilistic forecasting of irregularly sampled multivariate time series with missing values
		is crucial for decision making in various domains, including health care, astronomy, and climate.
		State-of-the-art methods estimate only
		marginal distributions of observations
		in single channels and at single timepoints,
		assuming a Gaussian distribution for the data.
		In this work, we propose a novel model, ProFITi 	
		using conditional normalizing flows to learn 
		multivariate conditional distribution:
		joint distribution of the future values of the time series
		conditioned on past observations and specific channels and timepoints,
		without assuming 
		any fixed shape of the underlying distribution.
		As model components, we introduce a novel invertible triangular attention layer
		and an invertible non-linear activation function on and onto the whole real line.
		Through extensive experiments on 4 real-world datasets,
		ProFITi demonstrates significant improvement,
		achieving an average log-likelihood gain of 2.0
		compared to the previous state-of-the-art method.
\end{abstract}

\section{Introduction}
\label{sec:intro}

Irregularly sampled multivariate time series with missing values (IMTS) 
are common in various real-world scenarios such as health, astronomy and climate.
Accurate forecasting of IMTS is important for decision-making,
but estimating uncertainty is crucial to avoid overconfidence.
State-of-the-art models applied to this task are Ordinary Differential Equations (ODE) based models~\cite{Schirmer2022.Modeling,DeBrouwer2019.GRUODEBayes,Bilos2021.Neural} 
which are 1) computationally inefficient, and 2) 
offer only marginal likelihoods.
In practice, joint or multivariate distributions
are desired to capture dependencies and study forecasting scenarios.
With joint distributions one can estimate the likelihood of
specific combinations of future variables ex.
likelihood of having rain and strong winds,
which marginal or point forecast models cannot deliver.

For this, we propose a novel conditional normalizing 
flow model called {\textbf{ProFITi}}, 
for \textbf{Pro}babilistic \textbf{F}orecasting of \textbf{I}rregularly sampled Multivariate \textbf{Ti}me series.
ProFITi is designed to learn conditional joint distributions.
We also propose two novel model components that can be used in 
flow models: a sorted invertible triangular attention layer, \textbf{\fc{}}, parametrized by conditioning input
to learn joint distributions,
and an invertible non-linear activation
function designed for flows,
\textbf{\fa{}}, that is on and onto whole real line.
ProFITi consists of several invertible blocks build
using \fc{} and $\fa{}$
functions.
Being a flow-based model, ProFITi can learn any
random conditional joint distribution,
while existing models \cite{Schirmer2022.Modeling,DeBrouwer2019.GRUODEBayes,Bilos2021.Neural}
learn only Gaussians.

Experiments on $4$ real-world IMTS datasets,
attest the superior performance of ProFITi.
Our contributions are:
\begin{enumerate}
	
	\item Introduced \textbf{ProFITi}, a novel and, to the best of our knowledge, first
	normalizing flow based probabilistic forecasting
	model for predicting 
	multivariate conditional distributions
	of irregularly sampled time series with missing values (Section~\ref{sec:profiti}).
	
	\item Proposed a novel invertible triangular attention layer, named \textbf{\stisatt},
	which enables target variables to interact and 
	capture dependencies within a conditional normalizing flow framework (Section~\ref{sec:inv_cond_flow}).

	\item Proposed a novel non-linear, invertible,
	differentiable activation function on and onto
	the whole real line, \textbf{\fa} (Section~\ref{sec:act}).
	This activation function can be used in normalizing flows.

	\item Conducted experiments on $4$
	IMTS datasets for normalized joint negative loglikelihood.
	On average, ProFITi provides a loglikelihood gain of $2.0$
	over the previously best model (Section~\ref{sec:exp}).
\end{enumerate}
Implementation: \url{github.com/yalavarthivk/ProFITi}

\section{Literature Review}
\label{sec:literature}

There have been multiple works dealing point forecasting
of irregular time series~\cite{Yalavarthi2024.GraFITi, Ansari2023.Neural,Chen2024.ContiFormer}.
Very few models provide uncertainty quantification 
for such forecasts.

\paragraph{Probabilistic Forecasting Models for IMTS.}
Probabilistic IMTS forecasting often relies on variational inference
or predicting distribution parameters.
Neural ODE models~\citep{Chen2018.Neural}
combine probabilistic latent states
with deterministic networks.
Other approaches like
  latent-ODE~\citep{Rubanova2019.Latent},
  GRU-ODE-Bayes~\citep{DeBrouwer2019.GRUODEBayes},
  Neural-Flows~\citep{Bilos2021.Neural}, and
  Continuous Recurrent Units (\citealp{Schirmer2022.Modeling})
provide only marginal distributions, no joint distributions.
Similarly, probabilistic interpolation models, such as HETVAE~\citep{Shukla2022.Heteroscedastic} and TripletFormer~\citep{Yalavarthi2023.Tripletformer}, also provide only Gaussian marginal distributions.
In contrast, Gaussian Process Regression
models (GPR; \citealp{Durichen2015.Multitask,Li2015.Classification,Li2016.Scalable,Bonilla2007.Multitask}) offer full joint posterior
distributions for forecasts, but 
struggle with the
computational demands on long time series due to the dense matrix
inversion operations.
All the models assume the data distribution is to be Gaussian and 
fail if the true distribution is different.

\paragraph{Normalizing Flows for variable input size.}
We deal with predicting distributions for variable many targets. 
This utilizes equivariant transformations, as shown in \citep{Bilos2021.Normalizing,Satorras2021.Equivariant,Liu2019.Graph}.
All the models apply continuous normalizing flows which require solving an ODE 
 driven by a neural network using a slow numerical integration process.
Additionally, they cannot incorporate conditioning inputs.

\paragraph{Conditioning Normalizing Flows.}
Learning conditional densities has been largely explored within computer vision \citep{Khorashadizadeh2023.Conditional,Winkler2019.Learning,AnanthaPadmanabha2021.Solving}.
They apply normalizing flow blocks such as affine transformations~\citep{Dinh2017.Density}, autoregressive transformations~\citep{Kingma2013.Autoencoding} or Sylvester flow blocks~\citep{vandenBerg2018.Sylvester}.
Often the conditioning input is appended to the target while passing through the flow layers as
shown in~\citep{Winkler2019.Learning}.
For continuous data representations only a few works 
exist~\citep{Kumar2020.VideoFlow,deBezenac2020.Normalizing,Rasul2021.Multivariate,Si2022.Autoregressive}.  
However, methods that deal with regular multivariate time series (such as \citealp{Rasul2021.Multivariate})
cannot handle IMTS due to its missing values.
We solve this by using invertible attention that allows flexible size.

\paragraph{Flows with Invertible Attention.}
To the best of our knowledge, there have been only two works that
develop invertible attention for Normalizing Flows.
Sukthanker et. al. \cite{Sukthanker2022.Generative} proposed an invertible attention by adding the
identity matrix to a softmax attention. However,
softmax yields only positive values in the attention matrix and does not learn
negative covariances.
Zha et. al. \cite{Zha2021.Invertible} introduced residual attention similar to
residual flows~\cite{Behrmann2019.Invertible} that
suffer from similar problems as residual flows
such as the lack of an explicit inverse making inference slow.
Additionally, computing determinants of dense attention
matrices has cubic complexity which is not desired.

\section{Problem Setting \&\ Analysis}
\label{sec:prilims}

\paragraph{The IMTS Forecasting Problem.}
\label{sec:imts_forec}

An \textbf{irregularly sampled multivariate times series with missing values}
(called briefly just \textbf{IMTS} in the following),
is a sequence
$x\obs = ((t_\tau, v_\tau))_{\tau=1:\mathcal{T}}$ where $v_\tau\in \{\R\cup \nan\}^C$
is an observation event at timepoint $t_\tau \in \R$.
$v_{\tau,c} \in \R$ indicates observed value and $\nan$
indicates a missing value. Horn et. al. \cite{Horn2020.Set} introduced set notation
where only observed values are considered and
missing values are ignored. For notational
continence, we use sequences whose order does not matter
to represent sets.

Now, 
$x\obs = \left((t_i\obs, c_i\obs, o_i\obs)\right)_{i=1:I}$
is a sequence of unique triples, 
where
$t_i\obs \in \R$ denotes the time,
$c_i\obs \in \{1,...,C\}$ the channel and
$o_i\obs \in \R$ the value of an observation,
$I\in\N$ the total number of observations across all channels
and $C\in\N$ the number of channels.

An \textbf{IMTS query} is a sequence
$x\qu = \left((t_k\qu, c_k\qu)\right)_{k=1:K}$
of just timepoints and channels (also unique), a sequence $y\in \R^K$ we call an \textbf{answer}. It is understood that
$y_k$ is the answer to the query $(t_k\qu, c_k\qu)$.

The \textbf{IMTS probabilistic forecasting problem} then
is, given a dataset
$\D^\text{train}:= \left(({x\obs}^n, {x\qu}^n, {y}^n)\right)_{n=1:N}$
of triples of time series, queries and answers from an unknown distribution $p$
(with earliest query timepoint is beyond the latest observed timepoint for series $n$, $\min_k {t\qu_k} > \max_i {t\obs_i}$),
to find a model $\hat p$ that maps each observation/query pair $(x\obs,x\qu)$
to a joint density over answers,
$\hat p(y_1, \ldots, y_k \mid x\obs, x\qu)$, such that the expected joint negative
log likelihood is minimal:
\begin{align*}
	\ell^\text{jNLL}(\hat p; p) := - {\mathbb E}_{(x\obs,x\qu,y)\sim p} \log \hat p(y \mid x\obs, x\qu)
\end{align*}
Please note, that the number $C$ of channels is fixed, but
the number $I$ of past observations and
the number $K$ of future observations queried
may vary over instances $(x\obs,x\qu,y)$.
If query sizes $K$ vary, instead of (joint) negative log likelihood one also can normalize
by query size to make numbers comparable over different series and limit
the influence of large queries, the \textbf{normalized joint negative
	log likelihood} njNLL:
\begin{align}
	\ell^\text{njNLL}(\hat p; p)
	& :=- \mathop{\mathbb{E}}_{(x\obs, x\qu, y)\sim p} \frac{1}{|y|} \log \hat p(y \mid x\obs, x\qu)
	\label{eq:njnl}
\end{align}
\paragraph{Problem Analysis and Characteristics.}
As the problem is not just an (unconditioned) density estimation problem,
but the distribution of the outputs depends on both, the past observations
and the queries, a \textbf{conditional density model} is required 
(\textbf{requirement 1}).

A crucial difference from many settings addressed in the related work~\cite{Schirmer2022.Modeling,Bilos2021.Neural,DeBrouwer2019.GRUODEBayes}
is that we look for probabilistic models of the \textbf{joint distribution} of all
queried observation values $(y_1, \ldots, y_K)$, not just at
the \textbf{single variable marginal distributions}
$p(y_k \mid x\obs, x\qu_k)$ (for $k=1{:}K$).
The problem of marginal distributions is a special case of our formulation
where all queries happen to have just one element (always $K=1$).
So for joint probabilistic forecasting of IMTS,
models need to output densities on a \textbf{variable number of
	variables} (\textbf{requirement 2}).

Furthermore, since we deal with the set representation of IMTS,
whenever two query elements get swapped,
a generative model should swap its output accordingly,
a density model should yield the same density value,
i.e.,
the model should be \textbf{permutation invariant} (\textbf{requirement 3}).
For any permutation $\pi$:
\begin{align}
	\hat p(y_1,\ldots,&y_K \mid x\obs,x\qu_1,\ldots,x\qu_K) 	= \nonumber \\
	& \hat p(y_{\pi(1)},\ldots,y_{\pi(K)} \mid x\obs,x\qu_{\pi(1)},\ldots,x\qu_{\pi(K)})
	\label{eq:invariant}
\end{align}

\section{Invariant Conditional Normalizing Flow Models}
\label{sec:inv_cond_flow}
\paragraph{Normalizing flows.}
Parametrizing a specific distribution such as the Gaussian is a 
simple and robust approach to probabilistic forecasting.
It can be added on top of any point forecasting model
(for marginal distributions or fixed-size queries at least).
However, such models are less suited for targets having a more
complex distribution.
Then typically normalizing flows are
used \cite{Rippel2013.Highdimensional,Papamakarios2021.Normalizing}.
A normalizing flow is an (unconditional) density model for 
variables $y\in\R^K$ consisting of
a simple base distribution,
typically a standard normal $p_Z(z):= \normal(z; 0_K, \id_{K\times K})$, and
an invertible, differentiable, parametrized map
$f(z; \theta): \R^K \to \R^K$; then
\begin{align}
	\hat p(y; \theta) :=
	p_Z(f^{-1}(y; \theta)) \left| \text{det}\left(\frac{\partial f^{-1}(y; \theta)}{\partial y}\right) \right|
	\label{eq:nf}
\end{align}
is a proper density, i.e., integrates to 1, and can be fitted to data minimizing
negative log likelihood via gradient descent algorithms.
A normalizing flow can be conditioned on predictor variables $x\in\R^M$
by simply making $f$ dependent on predictors $x$, too: $f(z; x, \theta)$ (satisfying \textbf{requirement 1}).
$f$ then has to be invertible w.r.t. $z$ for any $x$ and $\theta$ \citep{Trippe2018.Conditional}.

\paragraph{Invariant conditional normalizing flows.}
A conditional normalizing flow represents an invariant
conditional distribution in the sense of eq.~\ref{eq:invariant} (\textbf{requirement 3}),
if
i) its predictors $x$ also can be grouped into $K$ elements $x_1,\ldots,x_K$
and possibly common elements $x\com$:
$x =(x_1,\ldots,x_K,x\com)$, and
ii) its transformation $f$ is equivariant in stacked $x_{1:K}$ and $z_{1:k}$:
\begin{align}
	f(z^{\pi}; x_{1:K}^{\pi}, x\com, \theta)^{\pi^{-1}} & = f(z; x_{1:K}, x\com, \theta) \nonumber
	\\
	 & \forall \text{permutations } \pi
  \label{eq:conditional-nf}
\end{align}
where $z^{\pi}:= (z_{\pi(1)},\ldots,z_{\pi(K)})$ denotes a permuted vector.
We call this an \textbf{invariant conditional normalizing flow model}.
%
If $K$ is fixed, we call it \textbf{fixed size}, otherwise \textbf{dynamic size}. 
In our work, we consider $x_{1:K}$ as the embedding of $\xq_1,\ldots,\xq_K$
and $\xobs$ and ignore $x^\text{com}$ (see eq.~\ref{eq:grafiti}).

\paragraph{Invariant conditional normalizing flows via invertible attention.}
The primary choice for a dynamic size (\textbf{requirement 2}), equivariant, parametrized function
is attention (\satt; \citealp{Vaswani2017.Attention}):
\begin{align*}
	A(X_\text{Q}, X_\text{K}) &:= X_\text{Q}W_\text{Q}(X_\text{K}W_\text{K})^T, \quad
	\\ A^{\text{softmax}}(X_\text{Q}, X_\text{K}) &:= \text{softmax}(A(X_\text{Q}, X_\text{K}))
	\\   \satt(X_\text{Q}, X_\text{K}, X_\text{V}) & := A^{\text{softmax}}(X_\text{Q}, X_\text{K})\cdot X_\text{V}W_\text{V}
\end{align*} 
where $X_\text{Q}, X_\text{K}, X_\text{V}$ are query, key and value matrices,
$W_\text{Q},W_\text{K},W_\text{V}$ are parameter matrices (not depending on the number of rows of $X_\text{Q}, X_\text{K}, X_\text{V}$)
and the softmax is taken rowwise.

Self attention mechanism ($X_\text{Q} = X_\text{K} = X_\text{V}$)
has been used in the literature as is for unconditional
vector fields \citep{Kohler2020.Equivariant,Li2020.Exchangeable,Bilos2021.Normalizing}.
To be used in a conditional vector field, $X_\text{Q} = X_\text{K} = X$ 
will have to contain the condition elements $x_{1:K}$,
$X_\text{V} = Z$ contains the base samples $z_{1:K}$ and $W_\text{V} = 1$.
\begin{align}
	X := \left[\begin{array}{cc} x_1^T
		\\ \vdots
		\\ x_K^T
	\end{array}\right], \quad Z :=  \left[\begin{array}{cc} z_1
	\\ \vdots
	\\ z_K
\end{array}\right] \label{eq:xmatrix}
\end{align}
Now, we make attention matrix itself invertible.
To get \textbf{invertible attention (\isatt)}, we
regularize the attention matrix $A$ sufficiently to become invertible
(see Lemma 1 in appendix for proof)
\begin{align}
	A^{\text{reg}}(X) & := \frac{1}{\|A(X,X)\|_2+\epsilon} A(X, X) + \id      \label{eq:eps}
	\\ \isatt(Z, X) & := A^{\text{reg}}(X).Z
	\label{eq:isatt}
\end{align}
where $\epsilon>0$ is a hyperparameter.
We note that, $A^{\text{reg}}(X)$ is not a parameter of the model,
but computed from the conditioners $x$.
Our approach is different from iTrans attention \cite[fig. 17]{Sukthanker2022.Generative}
that makes attention invertible more easily via
$A^{\text{iTrans}}(X):= A^{\softmax}(X,X) + \id$ using the fact that the
spectral radius $\sigma(A^{\softmax}(X,X)) \le 1$,
but therefore is restricted to non-negative interaction weights.

The attention matrix $A^{\text{reg}}(X)$ will be dense in general and thus slow to
invert, taking $\bigO(K^3)$ operations. 
Following ideas for autoregressive flows and coupling layers, 
a triangular matrix would allow a much more
efficient inverse pass, as its determinant
can be computed in $\bigO(K)$ and linear systems
can be solved in $\bigO(K^2)$.
%
This does not restrict the expressivity of the model, as
due to the Knothe–Rosenblatt rearrangement~\citep{Villani2009.Optimal}
from optimal transport theory,
any two probability distributions on $\R^K$ can be transformed
into each other by flows with a locally triangular Jacobian.
%
Unfortunately, just masking the upper triangular part of the matrix will destroy the
equivariance of the model. We resort to the simplest way to make
a function equivariant: we sort the inputs before passing them into the layer and revert the outputs to the original ordering. We call this approach 
\textbf{sorted invertible triangular attention (\stisatt)}:
\begin{align}
	 \pi & := \argsort (x_1S,\ldots,x_KS) \label{eq:pi}
	\\   A^{\text{tri}}(X) & := \text{softplus-diag}(\text{lower-triang}(A(X,X))) +\epsilon\id \label{eq:Atri}
	\\ \stisatt(Z, X) & :=(A^{\text{tri}}(X^\pi) \cdot Z^\pi)^{\pi^{-1}} \label{eq:SITA}
\end{align}
where $\pi$ operates on the rows of $X$ and $Z$.
Softplus activation is applied to diagonal elements making them positive.
Sorting is a simple lexicographic sort along the dimensions of vector $x_kS$.
The matrix $S$ allows to specify a sorting criterion, e.g., a permutation
matrix.
Note that sorting is unique only when $x$ has unique elements.
{\em In practice, we compute $\pi$ from $x\qu$ instead of $x$, 
first sort by timepoint, and then by channel}.

\begin{exple}[Demonstration of sorting in SITA]
	Given $x\qu = ((1,2), (0,2), (2,1), (3,1), (0,1), (3,3))$ 
	where first and second elements in $\xq_k$ indicate queried time and 
	channel respectively.
	Assume $S = 
	\begin{psmallmatrix}
		1 & 0 \\ 
		0 & 1
	\end{psmallmatrix}$. Then
	\begin{align*}
		\pi	& = \argsort(x\qu_1S, \ldots, x\qu_5S) \\
		& = \argsort((1,2), (0,2), (2,1), (3,1), (0,1), (3,3)) \\
		& = (5,2,1,3,4,6)
	\end{align*}
\end{exple}

\section{Shiesh: A New Activation Function for Normalizing Flows}
\label{sec:act}

The transformation function $f$ of a normalizing flow usually is realized
as a stack of several simple functions.
As in any other neural network,
elementwise applications of a function, called activation functions, is one
of those layers that allows for non-linear transformations.
\begin{table}[t]
	\centering
	\caption{Properties of existing activation functions. \commentout{E1: Bijective, E2: Domain alignment, E3: Non-zero gradient}}
	
	\label{tab:acts}
	\begin{tabular}{l|ccc}
		\toprule
		Activation	& E1	& E2	&	E3\\
		\hline
		ReLU	&	$\cross$	&	$\cross$	&	$\cross$	\\
		Leaky-ReLU	&	$\checkmark$	&	$\checkmark$	& $\checkmark$	\\
		P-ReLU	&	$\checkmark$	&	$\checkmark$	& $\checkmark$	\\		
		ELU		&	$\checkmark$	&	$\cross$ &	$\checkmark$\\
		SELU	&	$\checkmark$	&	$\cross$	&	$\checkmark$\\
		GELU	&	$\cross$	&	$\cross$ & $\checkmark$	\\
		Tanh	&	$\checkmark$	&	$\cross$	&	$\checkmark$	\\
		Sigmoid	&	$\checkmark$	&	$\cross$	&	$\checkmark$	\\
		Tanh-shrink	&	$\checkmark$	&	$\checkmark$	&	$\cross$	 \\
		\hline
		\act	&	$\checkmark$	&	$\checkmark$	&	$\checkmark$	\\
		\bottomrule
	\end{tabular}
\end{table}
However, most common activation functions used in deep learning are not diffeomorphic and do not have both their domain and co-domain on the entire real line, making them unsuitable for normalizing flows. 
For instance:
\begin{itemize}
	\item \textbf{ReLU} is not invertible (E1)
	\item \textbf{ELU} cannot be used consistently throughout the layer stack because its output domain, \(\mathbb{R}^+\), does not span the entire real number line (E2)
	\item \textbf{Tanh-shrink} (\(\text{tanhshrink}(u) := u - \tanh(u)\)) is invertible and covers the entire real line, but it has a zero gradient at some points (e.g., at \(u = 0\)). This zero gradient makes it impossible to compute the inverse of the normalizing factor, \(\left| \text{det}\left(\frac{\partial f(u)}{\partial u}\right) \right|\), required for normalizing flows (E3)
\end{itemize}

To serve as a standalone layer in a normalizing flow, an activation function
must fulfill these three requirements:
\textbf{E1}. be invertible,
\textbf{E2}. cover the whole real line and
\textbf{E3}. have no zero gradients.
Out of all activation functions in the
pytorch library (V 2.2) only Leaky-ReLU and P-ReLU
meet all three requirements (see Table~\ref{tab:acts}). 
Leaky-ReLU and P-ReLU usually are used with a
slope on their negative branch being well less than 1, so
that stacking many of them might lead to small gradients
also causing problems for the normalizing constant of a normalizing flow.

To address these challenges, we propose a new activation function derived from unconstrained monotonic neural networks (UMNN; \citealp{Wehenkel2019.Unconstrained}).
UMNN have been proposed as versatile, learnable activation functions
for normalizing flows, being basically a continuous flow
for each scalar variable $u$ separately and a 
scalar field $g$ implemented
by a neural network: 
\begin{align}
	a(u) := v(1)
	 \text{ with } v : [0,1]\rightarrow\R \nonumber \\ 
	\text{ being the solution of } \; \; \frac{\partial v}{\partial\tau} = g(\tau, v(\tau)),
	\quad   v(0)  := u
	\label{eq:umnn}
\end{align}

In consequence, they suffer from the same
issues as any continuous normalizing flow: they are slow as they
require explicit integration of the underlying ODE.
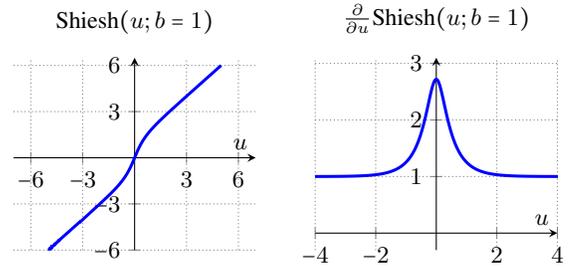
\begin{figure}[t]
	\centering
	\small
	\pgfkeys{/pgf/declare function={arcsinh(\x)=ln(\x+sqrt(\x^2+1));}}

\begin{tikzpicture}{\scale=0.8}
\begin{axis}[
	name=plot1,
    width=0.27\textwidth,
    axis lines=middle,
    xlabel={$u$},
    ymin=-6, ymax=6.5,
    xmin=-7, xmax=7,
    ytick={-6,-3, 0, 3, 6},
    xtick={-6,-3, 3, 6},
    title={$\act(u; b=1)$},
	grid=major,
	grid style={thin,densely dotted,black!50}]
\addplot[domain=-5:5, line width=0.4mm, samples=500, blue]{arcsinh(exp(1)*sinh(x))};
\end{axis}
\hspace{0.5cm}
\begin{axis}[
at=(plot1.right of south east), anchor=left of south west,
width=0.27\textwidth,
axis lines=middle,
xlabel={$u$},
ymin=-0.3, ymax=3.1,
xmin=-4, xmax=4,
ytick distance=1,
xtick={-4,-2, 0, 2, 4},
title={${\frac{\partial}{\partial u}\act(u; b=1)}$},
grid=major,
grid style={thin,densely dotted,black!50}]
\addplot[domain=-5:5, line width=0.4mm, samples=1000, blue]{exp(1)*cosh(x)/sqrt(1 + (exp(1)*sinh(x))^2)};
\end{axis}
\end{tikzpicture}
	\caption{(left) \act{} function, (right) partial derivative.}
	\label{fig:activation}
\end{figure}
Besides requirements E1--E3, activation functions will profit from
further desired properties:
\textbf{D1}. having an explicit inverse,
\textbf{D2}. having an explicit Jacobian and
\textbf{D3}. having a bounded gradient.
UMNN do not have desired property D1 
and provide no guarantees for property D3.

Instead of parameterizing the scalar field $g$ and learn it from data, we
make an educated guess and choose a specific function with
few parameters for which eq.~\ref{eq:umnn} becomes
explicitly solvable and requires no numerics at runtime: for the scalar field
$g(\tau, a; b):= \tanh(b \cdot a(\tau))$ the resulting ODE
\begin{align*}
	\frac{\partial v}{\partial\tau} = \tanh(b \cdot v(\tau)), 
	\quad   v(0)  := u
\end{align*}
has an explicit solution (Sec. G)
\begin{equation*}
	v(\tau; u,b) =  \frac{1}{b}\sinh^{-1}\big(e^{b\cdot \tau}\cdot \sinh(b\cdot u)\big)
\end{equation*}
yielding our activation function \act:
\begin{align}
	\act(u; b) \, := \, a(u) \, &:= \, v(1; u,b) \, \nonumber \\
	&= \, \frac{1}{b}\sinh^{-1}\big(e^{b} \sinh(b\cdot u)\big)
	\label{eq:fa}
\end{align}
being invertible, covering the whole real line and having no zero gradients (E1--E3)
and additionally
with analytical inverse and gradient (D1 and D2)
\begin{align}
	\begin{aligned}
	\fa^{-1}(u; b) & = \frac{1}{b}\sinh^{-1}\big(e^{-b}\cdot \sinh(b\cdot u)\big) 
	\\
	\frac{\partial}{\partial u} \act(u; b) & = 
	\frac{e^b \cosh(b\cdot u)}{\sqrt{1 + \big(e^b \sinh(b\cdot u)\big)^2}}
	\end{aligned}
\label{eq:jfa}
\end{align}
and bounded gradient (D3) in the range $(1, e^b]$ (see appendix G.4). Figure~\ref{fig:activation}
shows a function plot. In our experiments we fixed its parameter $b = 1$.

Since the $\fa$ is applied element-wise, it does not violate the Requirements 1, 2 and 3 in Section~\ref{sec:prilims}.:

\section{Overall ProFITi Model Architecture}
\label{sec:profiti}

\begin{figure}[t]
	\centering
	\small
		\begin{tikzpicture}[scale=0.9, transform shape]
	\node (xobs) at (-0.42,0) {$x\obs$};
	\node (xqu) at (0.2,0) {$x\qu$};
	\node (S) at (1.6,0) {};
	\node (y) at (1.8,0) {$y$};
	\node (S1) at (6.0,0) {$S$};
	\node (xqu1) at (6.8,0) {$x\qu$};
	\draw[dashed] (-0.6, -0.35) -- node[above] {(inputs)} (7,-0.35);
	\node [draw, rectangle, fill=cyan!30, rounded corners] (argsort) at (6.4, -0.8) {$\argsort$};
	
	\node[draw, rectangle, fill=yellow!30, rounded corners, below of=y, node distance=0.8cm] (sorty) {sort};
	\node[draw, rectangle, fill=yellow!30, rounded corners, below of=xqu, node distance=0.8cm] (sortx) {sort};
	\node[right of = sortx, node distance=0.8cm] (pi) {$\pi$}; 
	\node[draw, rectangle, fill=blue!30, rounded corners] (grafiti) at (-0.15,-1.6) {GraFITi};
	\node[draw, rectangle, fill=red!30, rounded corners, minimum width=2.1cm] (el0) at ([yshift=-1.6cm, xshift=-0.6cm]sorty) {$\text{EL}^{(0)}$};
	\node[draw, rectangle, fill=red!30, rounded corners, below of=el0, node distance=1.2cm] (f1) {profiti-block$^{(1)}$};
	\node[draw, rectangle, fill=red!30, rounded corners, below of=f1, node distance=1.2cm] (fl) {profiti-block$^{(l)}$};
	\node[draw, rectangle, fill=red!30, rounded corners, below of=fl, node distance=1.2cm] (fL) {profiti-block$^{(L)}$};
	\node[draw, rectangle, rounded corners, dotted, thick, label=60:{$f^{-1}$}, minimum width=2.3cm, minimum height=4.5cm] (profiti-block-block) at ([yshift=0.6cm]fl) {};
	\node[below of = fL, node distance=1.5cm] (z) {$z$};
	\node[below right of = z, node distance=1.2cm] (density) {$\hat{p}(y) := p_Z(z)\left|\text{det}\left(\frac{\partial z}{\partial y}\right)\right|$};
	
	\draw[dashed] ([xshift=-1.5cm, yshift=0.6cm]z.center) -- node[below] [xshift=0.8cm]{(outputs)} ([xshift=1.5cm, yshift=0.6cm]z.center);
	
	\node[draw, rectangle, fill=green!20, rounded corners, right of=fl, node distance=3cm] (ell) {$\text{EL}^{(l)}$};
	\node[draw, rectangle, fill=green!20, rounded corners, above of=ell, node distance=1cm] (sital) {$\stisatt^{(l)}$};
	\node[draw, rectangle, fill=green!20, rounded corners, below of=ell, node distance=1cm] (shiesh) {$\fa$};
	\node[draw, rectangle, fill=none] (atril) at ([xshift=2.3cm, yshift=0.3cm]sital) {$A^{\text{tri}^{(l)}}$};
	\node[draw, rectangle, fill=none, below of = atril, node distance=1.5cm] (scal) {$\text{NN}^{\text{sca}^{(l)}}$};
	\node[draw, rectangle, fill=none, below of = scal, node distance=1cm] (trsl) {$\text{NN}^{\text{trs}^{(l)}}$};
	\node[right of=atril, node distance=1.2cm, inner sep=0] (x) {\normalsize $\bigotimes$};
	\node[below of=x, node distance=1.5cm, inner sep=0] (dot) {\normalsize $\bigodot$};
	\node[below of=dot, node distance=1cm, inner sep=0] (plus) {\normalsize $\bigoplus$};
	\draw[-latex, thick] (S1) -- (argsort.120);
	\draw[-latex, thick] (xqu1) -- (argsort.60);
	\draw[-latex, thick] (argsort.south) -- node [yshift=-3mm]{$\pi$} ([yshift=-4mm]argsort.270);
	\draw[-latex, thick] (xobs.250) -- ([yshift=-1.1cm]xobs.250);
	\draw[-latex, thick] (sortx.250) -- ([yshift=-0.35cm]sortx.250);
	\draw[-latex, thick] (pi) -- (sortx);
	\draw[-latex, thick] (pi) -- (sorty);
	\draw[-latex, thick] (xqu) -- (sortx);
	\draw[-latex, thick] (y) -- (sorty);
	\draw[-latex, thick] (grafiti.south) |- node [xshift=-2mm, yshift=2mm]{$x$} (el0.west);
	\draw[-latex, thick] (grafiti.south) |-   (f1.west);
	\draw[-latex, dashed, thick] (grafiti.south) |-   (fl.west);
	\draw[-latex, dashed, thick] (grafiti.south) |-   (fL.west);
	\draw[-latex, thick] (fL.south) --   (z);
	\draw[-latex, thick] (sorty) |- +(0,-5mm) -|  (el0);
	\draw[-latex, thick] (el0) -- node [xshift=-3mm]{$y^{(1)}$} (f1);
	\draw[-latex, thick, dashed] (f1) -- node [xshift=-3mm]{$y^{(l)}$} (fl);
	\draw[-latex, thick, dashed] (fl) -- node [xshift=-3mm]{$y^{(L)}$} (fL);
	\draw[-latex, thick] ([yshift=4mm]sital.north) -- (sital.north);
	\draw[-latex, thick] (sital) -- (ell);
	\draw[-latex, thick] (ell) -- (shiesh);
	\draw[-latex, thick] (shiesh.south) -- ([yshift=-4mm]shiesh.south);
	\draw[-latex, thick] ([xshift=-10mm]sital.center) -- node [yshift=2mm]{$x$} (sital.west);
	\draw[-latex, thick] ([xshift=-10mm]ell.center) -- node [yshift=2mm]{$x$} (ell.west);
	\draw[-latex, thick] (atril) -- (x);
	\draw[-latex, thick] (scal) -- (dot);
	\draw[-latex, thick] (trsl) -- (plus);
	\draw[-latex, thick] (dot.south) -- (plus);
	\draw[-latex, thick] ([yshift=6mm]x.north) -- (x.north);
	\draw[-latex, thick] (x.south) -- (dot);
	\draw[-latex, thick] (plus.south) -- ([yshift=-6mm]plus.south);
	\draw[-latex, thick] ([xshift=-12mm]atril.center) -- node [yshift=2mm]{$x$} (atril.west);
	\draw[-latex, thick] ([xshift=-12mm]scal.center) -- node [yshift=2mm]{$x$} (scal.west);
	\draw[-latex, thick] ([xshift=-12mm]trsl.center) -- node [yshift=2mm]{$x$} (trsl.west);
	\draw[-latex, thick] (z) -- ([yshift=-0.4cm]z.south);
	\node[draw, rectangle, minimum height=3.4cm, minimum width=1.8cm, rounded corners] (profitiblock) at ([xshift=-2mm]ell) {};
	\node[draw, rectangle, minimum height=1.2cm, minimum width=2.7cm, rounded corners] (sitablock) at ([xshift=1mm]atril) {};
	\node[draw, rectangle, minimum height=1.8cm, minimum width=2.7cm, rounded corners] (ellblock) at ([yshift=-5mm, xshift=1mm]scal) {};
	\draw[-latex, dashed] (ell.east) -- (ellblock.north west);
	\draw[-latex, dashed] (ell.east) -- (ellblock.south west);
	\draw[-latex, dashed] (sital.east) -- (sitablock.north west);
	\draw[-latex, dashed] (sital.east) -- (sitablock.south west);
	\draw[-latex, dashed] (fl.east) -- (profitiblock.north west);
	\draw[-latex, dashed] (fl.east) -- (profitiblock.south west);
\end{tikzpicture}
	\caption{ProFITi architecture; $\bigotimes$: dot product, $\bigodot$: Hadamard product, $\bigoplus$: addition.
	Functions referred to their equation numbers: sort, $\argsort$ (eq. \ref{eq:pi}), GraFITi (eq. \ref{eq:grafiti}), SITA (eq.~\ref{eq:SITA}), EL (eq. \ref{eq:ft}), Shiesh (eq. \ref{eq:fa}).
	For efficiency, we perform, sorting only once directly on $x\qu$ and $y$.}
	\label{fig:my_figure}
\end{figure}
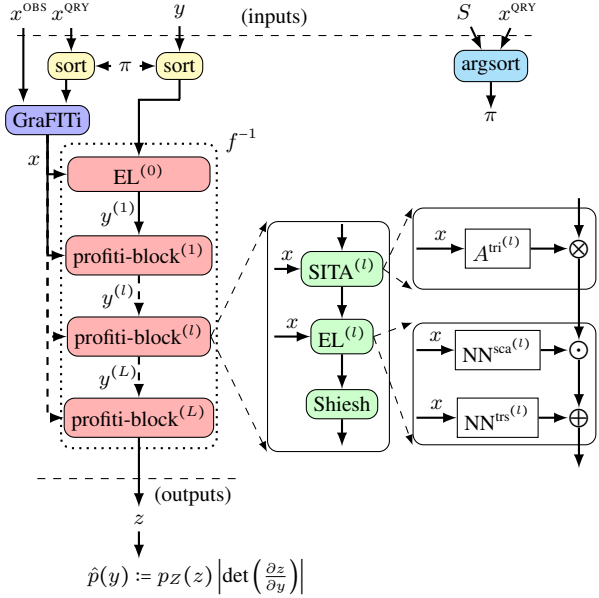

Invertible attention and the \act{} activation function systematically model
inter-dependencies between variables and non-linearity respectively, but do not
move the zero point. To accomplish the latter, we use a third layer called \textbf{elementwise linear transformation layer (\ft)}:
\begin{align}
	\ft{}(y_k; x_k) := y_k \cdot {\scale(x_k)} + \transl\left(x_k\right)
	\label{eq:ft}
\end{align}
where $\scale$ and $\transl$  are neural networks for scaling and translation.
$\scale$ is equipped with a $\exp(\tanh)$ output function
to make it positive and bounded, guaranteeing the inverse.
We combine 
all three layers from eq.~\ref{eq:isatt},  \ref{eq:fa}, and \ref{eq:ft}
to a block 
\begin{equation}
	\text{profiti-block}(y; x):= \act( \ft{}( \stisatt (y; x); x))
	\label{eq:profiti-block}
\end{equation}
and stack $L$ of those blocks to build the inverse transformation $f^{-1}$ of our conditional invertible flow model ProFITi.
We add a transformation layer with slope fixed to $1$ as initial encoding on the $y$-side
of the model.
See Figure~\ref{fig:my_figure} for an overview of its architecture.
As shown in the figure, for efficiency reasons we perform
sorting (eq.~\ref{eq:SITA}) only once
directly on the queries $x\qu$ and answers $y$.

\paragraph{Encoder for Query embedding.}
As discussed in Section~\ref{sec:inv_cond_flow}, for probabilistic time series forecasting we have to condition
on both, the past observations $x\obs$ and the queried time point/channel pairs
$x\qu$ of interest. While in principle any equivariant encoder
could be used, an encoder that leverages the relationships between those two pieces
of the conditioner is crucial. 
We use GraFITi~\citep{Yalavarthi2024.GraFITi}, a graph based equivariant point forecasting model for IMTS
that provides state-of-the-art performance (in terms of accuracy and efficiency) as encoder
\begin{align}
	(x_1,\ldots,x_K) := \qem(x\qu_1, \ldots, x\qu_K, x\obs) \label{eq:grafiti}
\end{align}
The Grafiti encoder is
trained end-to-end within the Profiti model, we did not pretrain it.

Note that for each query, other IMTS forecasting models yield a scalar, the predicted
value, not an embedding vector.
While it would be possible to use IMTS forecasting models as (scalar) encoders,
due to their limitations to a single dimension we did not follow up on this idea.

\paragraph{Training.}
We train the ProFITi model $\hat p$ 
for the normalized joint negative log-likelihood loss (njNLL; eq.~\ref{eq:njnl})
which written in terms of the transformation $f^{-1}(\cdot; \cdot; \theta)$
of the normalizing flow and its parameters $\theta$  yields:
\begin{align}
	\ell^\text{njNLL}(\theta) &:= \ell^\text{njNLL}(\hat p; p)
	\nonumber \\ 
	&= \mathop{\mathbb{-E}}_{(x\obs, x\qu, y)\sim p} \frac{1}{|y|}\log p_Z(f^{-1}(y; x\obs, x\qu; \theta)) \nonumber \\  
	& \quad \quad \Big| \text{det}\Big(\frac{\partial f^{-1}(y; x\obs, x\qu; \theta)}{\partial y} \Big) \Big|
	\label{eq:loss}
\end{align}

\section{Experiments}
\label{sec:exp}

\begin{table*}[t]
	\centering
	\caption{Normalized Joint Negative Log-likelihood (njNLL), lower the better, best in bold, OOM indicates out of memory error, $\uparrow$ shows imporvement in njNLL w.r.t. next best model.}
\label{tab:njnl}
\begin{tabular}{lrrrrrrrr}
	\toprule
	&	\multicolumn{1}{c}{USHCN} & $\frac{\text{time}}{\text{epoch}}$	& \multicolumn{1}{c}{Physioinet'12}	& $\frac{\text{time}}{\text{epoch}}$	&	\multicolumn{1}{c}{MIMIC-III}	&	$\frac{\text{time}}{\text{epoch}}$	&	\multicolumn{1}{c}{MIMIC-IV}	&	$\frac{\text{time}}{\text{epoch}}$	\\
	\midrule
	GPR		&	2.011$\pm$1.376	&	2s	&	1.367$\pm$0.074	&	35s	&		3.146$\pm$0.359	&	71s	&		2.789$\pm$0.057 &	227s\\
	HETVAE	&	{198.9$\pm$397.3}	&	1s&		0.561$\pm$0.012	&8s	&		{0.794$\pm$0.032}	&	8s	&	OOM	&	$-$\\
	GRU-ODE	&	0.766$\pm$0.159	&	100s&		0.501$\pm$0.001	&	155s	&		0.961$\pm$0.064		&	511s	&		0.823$\pm$0.318	&	1052s	\\
	Neural-Flows	&	0.775$\pm$0.152	&	21s&		{0.496$\pm$0.001}	&	34s	&		0.998$\pm$0.177	&	272s &		{0.689$\pm$0.087}	& 515s		\\
	CRU	&	0.761$\pm$0.191	&	35s&		0.741$\pm$0.001	&	40s	&		1.234$\pm$0.076	&	131s	&		OOM		&	$-$	\\
	GraFITi+	&	0.489$\pm$0.173	&	3s&		0.367$\pm$0.021	& 32s	&		0.721$\pm$0.053	&	80s&		0.287$\pm$0.040	&	84s	\\
	\midrule
	ProFITi (ours)	&	\bftab{-3.226$\pm$0.225}&	6s &		\bftab{-0.647$\pm$0.078} &	59s	&		\bftab{-0.377$\pm$0.032}	&	97s	&		\bftab{-1.777$\pm$0.066}		&		123s\\	
	\multicolumn{1}{c}{$\uparrow$}	&	\multicolumn{1}{c}{3.5}	&&	\multicolumn{1}{c}{1.0}	&&	\multicolumn{1}{c}{1.1}	&&	\multicolumn{1}{c}{2.1}	&\\
	\bottomrule
\end{tabular}
\end{table*}
\begin{table*}[t]
	\centering
	\caption{Results for Marginal Negative
		Log-likelihood (mNLL), lower the
		better. Best in bold and second best in italics.
		ProFITi\_marg is ProFITi trained for marginals.}
	\label{tab:mnll}
	\begin{tabular}{l r r r r}
		\toprule
		& \multicolumn{1}{c}{USHCN} & \multicolumn{1}{c}{Physionet'12}       & \multicolumn{1}{c}{MIMIC-III} & \multicolumn{1}{c}{MIMIC-IV} \\
		\midrule
		HETVAE         & 168.1$\pm$335.5 & 0.519$\pm$0.018 & 0.947$\pm$0.071 & OOM \\
		GRU-ODE        & 0.776$\pm$0.172 & 0.504$\pm$0.061 & 0.839$\pm$0.030 & 0.876$\pm$0.589 \\
		Neural-Flows   & 0.775$\pm$0.180 & 0.492$\pm$0.029 & 0.866$\pm$0.097 & 0.796$\pm$0.053	 \\
		CRU            & 0.762$\pm$0.180 & 0.931$\pm$0.019 & 1.209$\pm$0.044 & OOM \\
		GraFITi+       & \textit{0.462$\pm$0.122} & \textit{0.505$\pm$0.015} & \textit{0.657$\pm$0.040} & \textit{0.351$\pm$0.045} \\
		\midrule
		ProFITi\_marg (ours)	&	\textbf{-2.575$\pm$1.336}	&	\textbf{-0.368$\pm$0.033}	&	\textbf{0.092$\pm$0.036}	&	\textbf{-0.782$\pm$0.023}	\\
		\bottomrule
	\end{tabular}
\end{table*}
\subsection{Experiment for Joint Likelihoods}
\label{sec:exp_njnl}

\paragraph{Datasets. }
  We use $3$ publicly available real-world medical IMTS datasets:
  \textbf{MIMIC-III}~\citep{Johnson2016.MIMICIII},
  \textbf{MIMIC-IV}~\citep{Johnson2021.Mark}, and
  \textbf{Physionet'12}~\citep{Silva2012.Predicting}.
Datasets contain ICU patient records collected over $48$ hours.
The preprocessing procedures outlined
in~\cite{Yalavarthi2024.GraFITi,Bilos2021.Neural,DeBrouwer2019.GRUODEBayes}
were applied. 
Observations in Physionet'12, MIMIC-III and MIMIC-IV
were rounded to intervals of $1$ hr, $30$ min and $1$ min respectively.
We also evaluated on publicly available climate dataset
\textbf{USHCN}~\citep{Menne2015.United}.
It consists of climate data observed for 150 years from 1218 weather stations in USA.

\paragraph{Baseline Models. }
\textbf{ProFITi} is compared to $3$ probabilistic IMTS forecasting models:
  \textbf{CRU}~\citep{Schirmer2022.Modeling},
  \textbf{Neural-Flows}~\citep{Bilos2021.Neural}, and
  \textbf{GRU-ODE-Bayes}~\citep{DeBrouwer2019.GRUODEBayes}.
{To disentangle lifts originating from GraFITi (encoder) and those 
originating from ProFITi, we add \textbf{GraFITi+} as a baseline}.
GraFITi+ predicts an elementwise mean and variance
of a normal distribution.
As often interpolation models can be used seamlessly for forecasting, too,
we include \textbf{HETVAE}~\citep{Shukla2022.Heteroscedastic},
a state-of-the-art probabilistic interpolation model, for comparison.
Furthermore, we include Multi-task Gaussian Process Regression 
(\textbf{GPR};~\citealp{Durichen2015.Multitask})
as a baseline able to provide joint densities.

\paragraph{Protocol.}
We split the dataset into Train, Validation and Test in ratio 70:10:20, respectively.
We select the hyperparameters from $10$ random hyperparameter configurations
based on their validation performance.
We perform $5$ fold cross validation with the chosen hyperparameters.
Following~\cite{Bilos2021.Neural} and \cite{Yalavarthi2024.GraFITi},
we use the first $36$ hours as observation range and forecast the next $3$ time steps for medical datasets and
first $3$ years as observation range and forecast the next $3$ time steps for climate dataset.
Note that $3$ time steps meaning do not mean $3$ observations.
For example, in Physionet'12, $K$ varies between $3$ and $49$.
All models are implemented in PyTorch and run on GeForce RTX-3090 and 1080i GPUs.
We compare the models for Normalized Joint Negative Log-likelihood (njNLL) loss (eq.~\ref{eq:njnl}).
Except for GPR, and ProFITi, we take the average of the marginal negative log-likelihoods
of all the observations in a series to compute njNLL for that series.

Sampling-based metrics like the Energy Score \citep{Gneiting2007.Strictly} not only suffer from the curse of dimensionality but also evaluate multivariate distributions improperly \citep{Marcotte2023.Regionsb}. Similarly, Continuous Ranked Probability Score sum (CRPS-sum; \citealp{Rasul2021.Multivariate}) for multivariate probabilistic forecasting can be misled by simple noise models, where random forecasts may outperform state-of-the-art methods~\citep{Koochali2022.Random}.

\paragraph{Results. }
Table~\ref{tab:njnl} demonstrates the Normalized Joint Negative 
Log-likelihood (njNLL, lower the better) and run time per epoch for all the datasets.
Best results are presented in bold.
ProFITi outperforms all the prior approaches with significant margin on all the four datasets.
While GraFITi+ is the second-best performing model, 
ProFITi outperforms it by an average gain of 2.0 in njNLL.
We note that although GPR is predicting joint likelihoods, it performs poorly,
likely because of having very few parameters.
We note that njNLL of HETVAE is quite high for the USHCN dataset.
The reason is HETVAE predicted a very small
variance ($10^{-4}$) for $1$ sample 
whose predicted mean is farther from the target.
We do not provide results for CRU on MIMIC-IV 
as our GPU (48GB VRAM) gives out of memory errors.
The reason for such high likelihoods compared to the baseline models is
i.) not assuming a Gaussian underlying distribution and ii.) directly
predicting joint distributions (see Section~\ref{sec:ablations}).

\subsection{Auxiliary Experiments for Marginals}

Existing models in the related work \cite{DeBrouwer2019.GRUODEBayes} and \cite{Bilos2021.Neural} cannot predict
multivariate distributions, hence their evaluation was restricted
to Marginal Negative Log-likelihood (mNLL):
\begin{align}
	&\ell^{\text{mNLL}}(\hat p; \D^{\text{test}})  :=   - 
	\frac{\sum\limits_{(x\obs,x\qu,y) \in D^\text{test}}
		\sum\limits_{k=1}^{|y|}\log \hat{p}\left(y_k| x\obs, x\qu_k \right)}{\sum\limits_{(x\obs, x\qu, y) \in D^\text{test}} |y|}
	\label{eq:mnl}
\end{align}
We trained ProFITi only for marginals by removing \stisatt{} from the architecture and called ProFITi\_marg.
Results presented in Table~\ref{tab:mnll} shows that ProFITi\_marg outperforms baseline models again.

\begin{table*}[t]
	\centering
	\caption{Results for CRPS score,
		lower the better. 
		Best in bold and second best in italics}
	\label{tab:crps}
	\begin{tabular}{lcccc}
		\toprule
		& \multicolumn{1}{c}{USHCN} & \multicolumn{1}{c}{Physionet'12}       & \multicolumn{1}{c}{MIMIC-III} & \multicolumn{1}{c}{MIMIC-IV} \\
		\midrule
		HETVAE			&	0.229$\pm$0.017	&	0.278$\pm$0.001	&	0.359$\pm$0.009	&	OOM	\\
		GRU-ODE			&	0.313$\pm$0.012	&	0.278$\pm$0.001	&	0.308$\pm$0.005	&	0.281$\pm$0.004\\
		Neural-flows	&	0.306$\pm$0.028	&	0.277$\pm$0.003	&	0.308$\pm$0.004	&	0.281$\pm$0.004	\\
		CRU				&	0.247$\pm$0.010	&	0.363$\pm$0.002	&	0.410$\pm$0.005	&	OOM \\
		GraFITi+		&	\textit{0.222$\pm$0.011}	&	\textit{0.256$\pm$0.001}	&	\textit{0.279$\pm$0.006}	&	\textit{0.217$\pm$.005} \\
		\midrule
		ProFITi\_marg (ours)	&	\textbf{0.192$\pm$0.019}	&	\textbf{0.253$\pm$0.001}	&	\textbf{0.276$\pm$0.001}	&	\textbf{0.206$\pm$0.001} \\
		\bottomrule
	\end{tabular}
\end{table*}
\begin{table*}
	\centering
	\caption{Results for MSE, lower the better. Best in bold and second best in italics.}
	\label{tab:mse}
	\begin{tabular}{lcccc}
		\toprule
		& \multicolumn{1}{c}{USHCN} & \multicolumn{1}{c}{Physionet'12}       & \multicolumn{1}{c}{MIMIC-III} & \multicolumn{1}{c}{MIMIC-IV} \\
		\midrule
		HETVAE			&	0.298$\pm$0.073	&	0.304$\pm$0.001	&	0.523$\pm$0.055	&	OOM	\\
		GRU-ODE			&	0.410$\pm$0.106	&	0.329$\pm$0.004	&	0.479$\pm$0.044	&	0.365$\pm$0.012	\\
		Neural-Flows	&	0.424$\pm$0.110	&	0.331$\pm$0.006	&	0.479$\pm$0.045	&	0.374$\pm$0.017	\\
		CRU				&	\textit{0.290$\pm$0.060}	&	0.475$\pm$0.015	&	0.725$\pm$0.037	&	OOM	\\
		GraFITi			&	\textbf{0.256$\pm$0.027}	&	\textbf{0.286$\pm$0.001}	&	\textbf{0.401$\pm$0.028}	&	\textbf{0.233$\pm$0.005}	\\
		\midrule
		ProFITi\_marg (ours)	&	0.300$\pm$0.053	&	\textit{0.295$\pm$0.002}	&	\textit{0.443$\pm$0.028}	&	\textit{0.246$\pm$0.004}	\\
		\bottomrule
	\end{tabular}
\end{table*}

In addition to mNLL, we also compare the models in terms of
CRPS score, a metric for marginal probabilistic forecasting
in Table~\ref{tab:crps}; and MSE, a metric for point forecasting in Table~\ref{tab:mse}.
While all the baseline models can compute CRPS and MSE
explicitly from Gaussian parameters,
ProFITi requires sampling.
For this, we randomly sampled $100$ samples.
We computed MSE from robustious mean of samples 
which is the mean of the samples with outliers removed.

ProFITi\_marg outperforms all the baseline models in terms of CRPS score.
On the other hand, for point forecasting, ProFITi\_marg is the second best model in 
terms of MSE, and GraFITi remains the best.
While we leverage GraFiTi as the encoder for ProFITi,
ProFITi incorporates various components specifically
designed to predict distributions, even if this sacrifices
some point forecast accuracy.
Also, GraFITi is trained to predict a Gaussian distribution, and the Mean Squared Error (MSE) is probabilistically equivalent to the negative log-likelihood of a Gaussian distribution with a fixed variance. This probabilistic interpretation can lead to better performance compared to models that are not specifically trained for Gaussian distributions.
Additionally, models for uncertainty quantification often 
suffer from slightly worse point forecasts, as observed in several
studies~\cite{Lakshminarayanan2017.Simple,Seitzer2021.Pitfalls,Shukla2022.Heteroscedastic}.
In the domain of probabilistic forecasting, the primary metric of interest is
negative log-likelihood, and ProFITi demonstrates superior performance.

\subsection{Ablation studies: Varying model components}
\label{sec:ablations}
\begin{table}[h]
	\centering
	\small
	\captionof{table}{Varying model components; Metric: njNLL;
		ProFITi-A+B: component A is removed and B is added.}
	\label{tab:abls}
	\begin{tabular}{l r}
		\toprule
		Model &	Physionet2012	\\
		\midrule
		ProFITi &	{-0.647$\pm$0.078}	\\
		ProFITi-$\fc$	&	-0.470$\pm$0.017 \\
		ProFITi-$\fa$	&	0.285$\pm$0.061	\\
		ProFITi-$\fc$-$\fa$	&	0.372$\pm$0.021\\
		\midrule
		ProFITi-$\fa$+PReLU	&	0.384$\pm$0.060	\\
		ProFITi-$\fa$+LReLU	&	\multicolumn{1}{c}{NaN Error}	\\
		ProFITi-$A^\text{tri}$+$A^\text{iTrans}$	&	-0.199$\pm$0.141\\
		ProFITi-$A^\text{tri}$+$A^\text{reg}$	&	-0.778$\pm$0.016	\\ 
		\bottomrule
	\end{tabular}
\end{table}

To analyze ProFITi's superior performance, we conduct an ablation study on Physionet'12 (Table~\ref{tab:abls}). The \fa{} activation significantly improves performance by enabling learning of non-Gaussian distributions (compare ProFITi and ProFITi-\fa{}). Similarly, learning joint distributions (ProFITi) outperforms ProFITi-\fc{} in njNLL. ProFITi-\fc-\fa{}, which learns only Gaussian marginals, performs worse than ProFITi. Replacing \fa{} with PReLU (ProFITi-\fa+PReLU) degrades performance, and using Leaky-ReLU (LReLU) results in small Jacobians and vanishing gradients.
The $A^\text{iTrans}$ (ProFITi-$A^\text{tri}+A^\text{iTrans}$) variant performs poorly due to its limitation to positive covariances. While ProFITi with $A^\text{reg}$ or $A^\text{tri}$ shows similar results, $A^\text{reg}$ faces scalability challenges: computing the determinant of a full attention matrix has $\bigO(K^3)$ complexity, compared to $\bigO(K)$ for triangular matrices. Additionally, $A^\text{reg}$ underperforms for longer forecast horizons
(see appendix H.2).
Further experiments on varying observation and forecast horizons and sparsity levels are provided in appendix H.

\section*{Conclusions}
In this work, we propose a novel model ProFITi for probabilistic forecasting
of irregularly sampled multivariate time series
with missing values using conditioning normalizing flows.
ProFITi is designed to learn multivariate conditional distribution of varying
length sequences.
To the best of our knowledge, ProFITi is the first normalizing
flow based model that can handle irregularly sampled time series with missing values.
We propose two novel model components, sorted invertible triangular attention
and \fa{} activation function
in order to learn any random target distribution.
Our experiments on four real-world datasets demonstrate that
ProFITi provides significantly better likelihoods than 
existing models.

\section*{Acknowledgments}
\begin{small}
	This work was supported by the Federal Ministry for Economic Affairs and Climate Action (BMWK), Germany, within the framework of the IIP-Ecosphere project (project number: 01MK20006D); co-funded by the Lower Saxony Ministry of Science and Culture under grant number ZN3492 within the Lower Saxony ``Vorab'' of the Volkswagen Foundation and supported by the Center for Digital Innovations (ZDIN); and also by the German Federal Ministry of
	Education and Research (BMBF) through the Program ”International Future Labs
	for Artificial Intelligence" (Grant 1DD20002A -- KIWI-BioLab)”
\end{small}

\bibliography{icml_24_profiti}

\bibliographystyle{aaai25}

%
\appendix
\begin{table*}[h]
	\centering
	\small
	\caption{Statistics of the datasets used our experiments. Sparsity means the percentage of missing observations in the time series. Time Sparsity means the percentage of time steps missing after discretizing the time series.}
	\label{tab:dset}
	\begin{tabular}{lcccccc}
		\toprule
		Name & \#Samples & \#Chann. & Max. len. & Max. Obs.& Sparsity	& Time Sparsity	\\
		\hline
		USHCN	&	1100	&	5	&	290	&	320	&	$77.9\%$	&	$84.3\%$\\
		Physionet'12 & 12,000 & 37 & 48 & 520 &$85.7\%$&	$4.4\%$\\
		MIMIC-III & 21,000& 96 & 96 & 710 &$94.2\%$	&	$72.9\%$\\
		MIMIC-IV & 18,000 & 102 & 710 & 1340 &$97.8\%$ &	$94.9\%$	\\
		\bottomrule
	\end{tabular}
\end{table*}

\section{Dataset Details}
\label{sec:data_supp}
Four datasets are used for evaluating the proposed model.
Basic statistics of the datasets is provided in Table~\ref{tab:dset}.

\paragraph{Physionet2012~\cite{Silva2012.Predicting}} encompasses the medical records of 12,000 patients 
who were hospitalized in the ICU.
During the initial 48 hours of their admission, 37 vital signs were measured.
We follow the protocol used in previous studies~\cite{Che2018.Recurrent,Cao2018.Brits,Tashiro2021.CSDI,Yalavarthi2024.GraFITi}.
After pre-processing, dataset consists of hourly observations making a total of up to 48 observations in each series.

\paragraph{MIMIC-III~\cite{Johnson2016.MIMICIII}} constitutes a medical dataset 
containing data from ICU patients admitted to Beth Israeli Hospital. 
96 different variables from a cohort of 18,000 patients were observed over an approximately 48-hour period.
Following the preprocessing procedures outlined in \citep{Bilos2021.Neural,DeBrouwer2019.GRUODEBayes,Yalavarthi2024.GraFITi}, we rounded the observations to 30-minute intervals.

\paragraph{MIMIC-IV~\cite{Johnson2021.Mark}} is an extension of the MIMIC-III database,
incorporating data from around 18,000 patients admitted to the ICU at a tertiary academic medical center in Boston.
Here, 102 variables are monitored. We followed the preprocessing steps of~\citep{Bilos2021.Neural,Yalavarthi2024.GraFITi}
and rounded the observations to 1 minute interval.

\paragraph{USHCN~\cite{Menne2015.United}} is a climate dataset consists of the measurements of
5 variables (daily temperatures, precipitation and snow) observed
over 150 years from 1218 meteorological stations in the USA. We
followed the same pre-processing steps given in \citep{DeBrouwer2019.GRUODEBayes,Yalavarthi2024.GraFITi} and selected a subset of 1114 stations and an observation window of 4
years (1996-2000).

\section{Literature review extended}
\label{sec:lit_rev_ext}

We provide comparison of various models that are related to ProFITi are presented in Table~\ref{tab:related}.
Density estimation models like RealNVP~\cite{Dinh2017.Density},
Inverse Autoregressive Flows~\cite{Kingma2016.Improved},
Selvester Flows~\cite{vandenBerg2018.Sylvester},
Residual Flow~\cite{Behrmann2019.Invertible},
and Graphical Normalizing Flows~\cite{Wehenkel2021.Graphical}
are density estimation models. They cannot be applied to
IMTS.
Conditional normalizing flows cab be built using the above density estimation models.
Additionally some important conditional normalizing flows exist such as
Cond. Normalizing Flow~\cite{Winkler2019.Learning},
Attn. Flow~\cite{Sukthanker2022.Generative},
or Inv. Dot. Attention Flow~\cite{Zha2021.Invertible}.
Generally, these are applied only to fixed length vector or image datasets.
They cannot be extended to IMTS.
Continuous normalizing flows such as
E(N)~\cite{Satorras2021.Equivarianta},
GNF~\cite{Liu2019.Graph},
or SNF~\cite{Bilos2021.Normalizing}
can handle varying number of variables and permutation invariant.
However, they are applied for conditional inputs.
Normalizing flows for time series is relatively new area of research with a few works.
MAF~\cite{Rasul2021.Multivariate},
CTFP~\cite{Deng2020.Modeling},
NKF~\cite{deBezenac2020.Normalizing},
and QFR~\cite{Si2022.Autoregressive}
are applied for probabilistic time series forecasting.
They can handle varying length sequences of regularly sampled and fully observed time series.
However, they cannot handle missing values in time series, and also not permutation invariant.
Finally, IMTS models such as GRU-ODE~\cite{DeBrouwer2019.GRUODEBayes},
NeuralFlows~\cite{Bilos2021.Neural},
CRU~\cite{Schirmer2022.Modeling}
can predict only marginal distributions assuming the underlying data distribution is Gaussian.
HETVAE is a probabilistic interpolation model which also predicts only Gaussian marginals.
On the other hand GPR~\cite{Durichen2015.Multitask} predict joint distributions
but again restricted to Gaussians.
GraFITi~\cite{Yalavarthi2024.GraFITi} is a point forecasting model and cannot predict any probability distributions.

\begin{table*}
	\centering
	\fontsize{8}{9}\selectfont
	\caption{Summary of Important models that 1. can be applied to Time Series with irregular sampling (Irreg. Samp.), or missing values (Miss. Vals.), 2. can predict marginal distributions (Marg. Dist.) or joint distributions (Joint Dist.), 3. can learn on conditional densities (Condition), 4. density of sequences with variable lengths (Dynamic) or 5. Permutation Invariant (Perm. Inv.). Parametric distributions are denoted with (Param).}
	\label{tab:related}
	\begin{tabular}{l|ccccccc}
		\hline
		Model & Irreg Samp & Miss Vals & Marg Dist & Joint Dist & Cond & Dynam. & Perm. Inv \\ 
		&	& 	& 	& 	& (Req. 1)	& (Req. 2)	& (Req. 3) \\ \hline
		GRU-ODE~\citep{DeBrouwer2019.GRUODEBayes} & $\checkmark$ & $\checkmark$ & (Param) & $\cross$ & $\checkmark$ & $\checkmark$ & $\checkmark$ \\ 
		Neural Flows~\citep{Bilos2021.Neural} & $\checkmark$ & $\checkmark$ & (Param) & $\cross$ & $\checkmark$ & $\checkmark$ & $\checkmark$ \\ 
		CRU~\citep{Schirmer2022.Modeling} & $\checkmark$ & $\checkmark$ & (Param) & $\cross$ & $\checkmark$ & $\checkmark$ & $\checkmark$ \\ 
		GPR~\citep{Durichen2015.Multitask}	& $\checkmark$ & $\checkmark$ & (Param) & (Param) & $\checkmark$ & $\checkmark$ & $\checkmark$ \\ 
		HETVAE~\cite{Shukla2022.Heteroscedastic} & $\checkmark$ & $\checkmark$ & (Param) & $\cross$ & $\checkmark$ & $\checkmark$ & $\checkmark$ \\ 
		GraFITi~\citep{Yalavarthi2024.GraFITi} & $\checkmark$ & $\checkmark$ & $\cross$ & $\cross$ & $\checkmark$ & $\checkmark$ & $\checkmark$ \\ 
		\hline
		RealNVP~\citep{Dinh2017.Density} & $\cross$ & $\cross$ & $\cross$ & $\checkmark$ & $\cross$ & $\cross$ & $\cross$ \\ 
		Inv. Autoreg~\citep{Kingma2016.Improved} & $\cross$ & $\cross$ & $\cross$ & $\checkmark$ & $\cross$ & $\cross$ & $\cross$ \\ 
		Selv. Flow~\citep{vandenBerg2018.Sylvester} & $\cross$ & $\cross$ & $\cross$ & $\checkmark$ & $\cross$ & $\cross$ & $\cross$ \\ 
		Residual Flow~\citep{Behrmann2019.Invertible}	&$\cross$ & $\cross$ & $\cross$ & $\checkmark$ & $\cross$ & $\cross$ & $\cross$\\
		Graphical~\citep{Wehenkel2021.Graphical} & $\cross$ & $\cross$ & $\cross$ & $\checkmark$ & $\cross$ & $\cross$ & $\cross$ \\ 
		\hline
		Cond. NF~\cite{Winkler2019.Learning}	& $\cross$ & $\cross$ & $\cross$ & $\checkmark$ & $\checkmark$ & $\cross$ & $\cross$ \\
		Attn. Flow~\cite{Sukthanker2022.Generative}	&  $\cross$ & $\cross$ & $\cross$ & $\checkmark$ & $\checkmark$ & $\cross$ & $\cross$ \\
		Inv. Dot. Attn~\cite{Zha2021.Invertible} &  $\cross$ & $\cross$ & $\cross$ & $\checkmark$ & $\checkmark$ & $\cross$ & $\cross$ \\
		\hline 
		E(N)~\citep{Satorras2021.Equivarianta} & $\cross$ & $\cross$ & $\cross$ & $\checkmark$ & $\cross$ & $\checkmark$ & $\checkmark$ \\ 
		GNF~\citep{Liu2019.Graph} & $\cross$ & $\cross$ & $\cross$ & $\checkmark$ & $\cross$ & $\checkmark$ & $\checkmark$ \\ 
		SNF~\citep{Bilos2021.Normalizing} & $\cross$ & $\cross$ & $\cross$ & $\checkmark$ & $\cross$ & $\checkmark$ & $\checkmark$ \\ 
		\hline
		MAF~\citep{Rasul2021.Multivariate} & $\cross$ & $\cross$ & $\checkmark$ & $\checkmark$ & $\checkmark$ & (Time) & $\cross$ \\ 
		CTFP~\citep{Deng2020.Modeling} & $\checkmark$ & $\cross$ & $\checkmark$ & $\checkmark$ & $\checkmark$ & (Time) & $\cross$ \\ 
		NKF~\citep{deBezenac2020.Normalizing} & $\cross$ & $\cross$ & $\cross$ & $\checkmark$ & $\checkmark$ & (Time) & $\cross$ \\ 
		QFR~\citep{Si2022.Autoregressive} & $\cross$ & $\cross$ & $\cross$ & $\checkmark$ & $\checkmark$ & (Time) & $\cross$ \\ 
		\hline 
		ProFITi (ours) & $\checkmark$ & $\checkmark$ & $\checkmark$ & $\checkmark$ & $\checkmark$ & $\checkmark$ & $\checkmark$ \\
		\bottomrule
	\end{tabular}
\end{table*}

\section{SITA Examples}
\label{sec:sort_expl}

We provide further examples for implementing SITA here:

\begin{exple}[Demonstration of $S$ and $\pi$ for SITA, sort by channel id. followed by timepoint]
	Given $x\qu = ((1,2), (0,2), (2,1), (3,1), (0,1), (3,3))$ 
	where first and second elements in $x\qu_k$ indicates time and 
	channel respectively.
	Assume $S = 
	\begin{psmallmatrix}
		0 & 1 \\ 
		1 & 0
	\end{psmallmatrix}$. Then
	\begin{align*}
		\pi	& = \argsort(x\qu_1S, \ldots, x\qu_5S) \\
				& = \argsort((2,1), (2,0), (1,2), (1,3), (1,0), (3,3)) \\
				& = (5,3,4,2,1,6)
	\end{align*}
	Here, $S$ helps to sort the $x\qu$ first by channel and then by time.
\end{exple}

\begin{exple}[Demonstration of $S$ and $\pi$ for SITA, sort by timepoint in descending order followed by channel id. in ascending order]
	Given $x\qu = ((1,2), (0,2), (2,1), (3,1), (0,1), (3,3))$ 
	where first and second elements in $x_k$ indicates time and 
	channel respectively.
	Assume $S = 
	\begin{psmallmatrix}
		-1 & 0 \\ 
		0 & 1
	\end{psmallmatrix}$. Then
	\begin{align*}
		\pi	& = \argsort(x\qu_1S, \ldots, x\qu_5S) \\
		& = \argsort((-1,2), (0,2), (-2,1), (-3,1), (0,1), (-3,3)) \\
		& = (4, 6, 3, 1, 5, 2)
	\end{align*}
	Here, $S$ helps to sort the $x\qu$ first by time in descending order and then by channel in ascending order.
	Finally,
	${x}^\pi = ((3,1), (3,3), (2,1), (1,2), (0,1), (0,2))$.
\end{exple}

\begin{exple}[Demonstration of $S$ and $\pi$ for SITA, sort by timepoint followed by altered order of channel id.]
	Given $x\qu = ((1,2), (0,2), (2,1), (3,1), (0,1), (3,3))$ 
	where first and second elements in $x\qu_k$ indicates time and 
	channel respectively.
	Assume $S = 
	\begin{psmallmatrix}
		1 & 0 \\ 
		0 & f(\cdot)
	\end{psmallmatrix}$.
	Whenever we encounter a function in matrix, we perform function operation instead of product.
	$f(\cdot)$ alters the channel index.
	\begin{align*}
		f(1) = 3 \\
		f(2) = 1 \\
		f(3) = 2 \\
	\end{align*}
Then
	\begin{align*}
		\pi	& = \argsort(x\qu_1S, \ldots, x\qu_5S) \\
		& = \argsort((1,1), (0,1), (2,3), (3,3), (0,3), (3,2)) \\
		& = (2,5,1,3,6,4)
	\end{align*}
	Here, $S$ helps to sort the $x\qu$ first by time in descending order and then by channel in ascending order.
	Finally,
	${x\qu}^\pi = ((0,2), (0,1), (1,2), (2,1), (3,3), (3,1))$.
\end{exple}

\section{Invertibility of $A^\textnormal{reg}$}
\label{sec:Areg_invert_proof}

We prove that $A^\text{reg}$ presented in Section~\ref{sec:inv_cond_flow} is invertible.
\begin{lma}	For any $K\times K$ matrix $A$ and $\epsilon>0$, the matrix $\mathbb{I}_K + \frac{1}{\|A\|_2 + \epsilon}A$ is invertible. Here, $\|A\|_2\coloneqq\max\limits_{x\ne 0} \frac{\|Ax\|_2}{\|x\|}$ denotes the spectral norm.
	\label{lma:Areg}
\end{lma}
\begin{proof}
	Assume it was not the case. Then there exists a non-zero vector $x$ such that $(\mathbb{I}_K + \frac{1}{\|A\|_2 + \epsilon}A)x=0$. But then
	$(\|A\|_2 + \epsilon)x = -Ax$, and taking the norm on both sides and rearranging yields $\|A\|_2\ge \frac{ \|Ax\|_2}{\|x\|_2} = \|A\|_2 + \epsilon > \|A\|_2$, contradiction! Hence the lemma.
\end{proof}

\commentout{
\section{Invariant conditional normalizing flows via continuous flows.}

Invariant conditional normalizing flow models have been developed
in the literature based on the \textbf{continuous flow} approach
\citep{Chen2018.Neural,Grathwohl2019.Scalable}, 
where the transformation $f$ is specified implicitly by an ordinary
differential equation with time-dependent vector
field $g: [0, 1]\times\R^K\rightarrow\R^K$:
\begin{align}
	\begin{aligned}
		f^{-1}(z) := v(1) \quad \text{ with } v : [0,1]\rightarrow\R^K \\
		\text{ being the solution of }
		\frac{\partial v}{\partial\tau} = g(\tau, v(\tau)),
		\quad   v(0)  := z
	\end{aligned}
	\label{eq:cnf}
\end{align}
$\tau$ often is called virtual time to clearly distinguish it from time as
an input variable. 
The vector field $g$ is represented by a parametrized function $g(\tau, v; \theta)$
and then can be learnt. Continuous flow models can be made conditional by
simply adding the predictors $x$ to the inputs of the vector field, too:
$g(\tau, v; x, \theta)$. 
Unconditional structured continuous flow models
can be made permutation invariant by simply making
the vector field permutation equivariant
\citep{Kohler2020.Equivariant,Li2020.Exchangeable,Bilos2021.Normalizing}:
$g(\tau, v^{\pi}; \theta)^{\pi^{-1}} =   g(\tau, v; \theta)$.
To make \textsl{conditional} structured continuous flow models
permutation invariant, the vector field has to be
\textbf{jointly permutation equivariant} in outputs $v$ and predictors $x$:
\begin{align*}
	g(\tau, v^{\pi}; x^{\pi}, \theta)^{\pi^{-1}} & =   g(\tau, v; x, \theta)
\end{align*}
}

\section{\act{} activation function}

\subsection{Solving ODE}
\label{sec:ode_solve}

The differential equation \( \frac{dv(\tau)}{d\tau} = \tanh(bv(\tau)), \quad v(0) := u \)
can be solved by separation of variables. However, we can also proceed as 
follows by multiplying the equation with $b\cosh(b\cdot v(\tau))$:
\begin{align*}
	&&  b \cosh(b v(\tau)) \frac{dv(\tau)}{d\tau}& = b \sinh(b  v(\tau))\\
	&\Leftrightarrow& \frac{d \sinh(b v(\tau))}{d\tau} &= b \sinh(bv(\tau)) \\
	&\Leftrightarrow& \sinh(b v(\tau))  &= C e^{b\tau}  \quad \text{for some $C$}\\
	&\Leftrightarrow& v(\tau) &=  \frac{1}{b} \sinh^{-1}(C e^{b\tau})  \quad \text{for some $C$}\; .
\end{align*}
The initial condition yields $C =  \sinh(bu)$

\subsection{Invertibility of \act{}}

A function $F:\R\to\R$ is invertible if it is strictly monotonically increasing.

\begin{thrm}
Function $\fa(u; b) = \frac{1}{b}\sinh^{-1}(e^{b} \sinh(b\cdot u))$ is strictly monotonically increasing for $u \in \R$.
\end{thrm}
\begin{proof}
	A function is strictly monotonically increasing if its first derivate is always positive. 
	From eq.~\ref{eq:jfa}, $\frac{\partial}{\partial u}\fa(u; b) := \frac{e^{b} \cosh(b\cdot u)}{\sqrt{1 + \left(e^{b\cdot \tau} \sinh(b\cdot u)\right)^2}}$.
	We known that $e^{b\cdot\tau}$ and $\cosh(u)$ are always positive hence
	$\frac{\partial }{\partial u}\fa(u; b)$ is always positive.
\end{proof}

\begin{figure}[t]
	\centering
	\captionsetup[subfigure]{justification=centering}
	\begin{tikzpicture}[scale=0.7, every node/.style={scale=0.8}]
		\draw[fill= blue, blue] (0,0) circle (0.08cm);
		\node[anchor=west] at (0.05,0) { $b=1$};
		
		\draw[fill=black!30!green, orange] (2,0) circle (0.08cm);
		\node[anchor=west] at (2.05,0) {$b = 2$};
		
		\draw[fill=green, green] (4,0) circle (0.08cm);
		\node[anchor=west] at (4.05,0) {$b = 3$};
	\end{tikzpicture}
	
	\includegraphics[width=0.6\columnwidth]{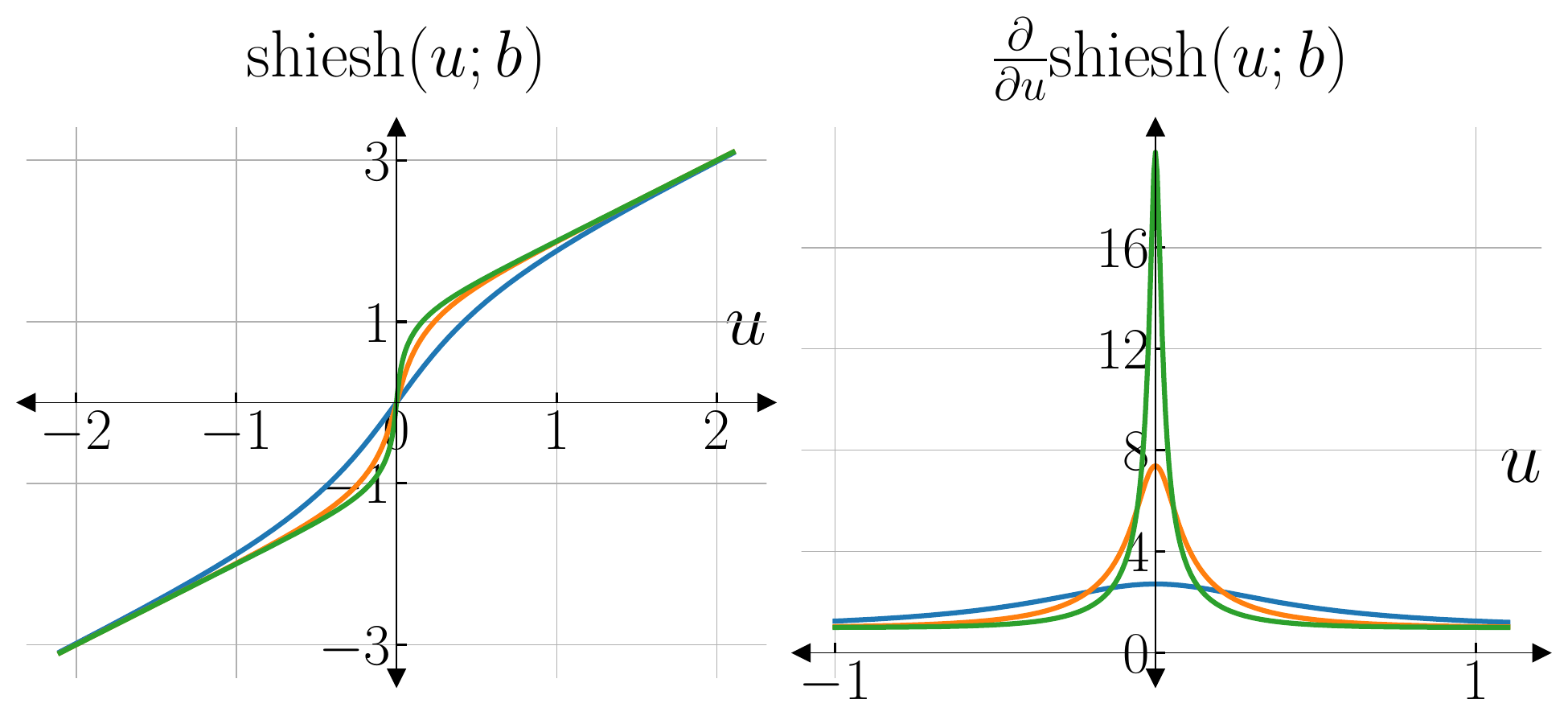}
	\caption{Demonstration of $\fa$ activation function with varying $b$.}
	\label{fig:supp_act}
\end{figure}
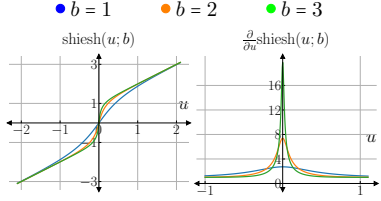

\subsection{Implementation details}
Implementing \fa{} on the entire $\R$ will have numerical overflow. Hence, we implement it in piece-wise manner. In this work, we are interested in $b > 0$ and show all the derivations for it.

With $\sinh(x) = \frac{e^x - e^{-x}}{2}$ and $\sinh^{-1}(x) = \log(x + \sqrt{1+x^2})$
$\fa$ can be rewritten as follows:
\begin{align*}
	\fa(u; b) &:= \frac{1}{b}\sinh^{-1}\big(\exp(b) \cdot \sinh(b\cdot u)\big) \\
			& = \frac{1}{b} \log\Big( \exp(b) \cdot \sinh(b\cdot u) \\ & \quad + \sqrt{1 + \big(\exp(b) \cdot \sinh(b\cdot u) \big)^2} \Big) \\
			& =  \frac{1}{b}\log\Bigg( \Big(\exp(b) \cdot \frac{\exp(b\cdot u)-\exp(-b\cdot u)}{2} \Big) \\ 
			& + \sqrt{1 + \Big(\exp(b) \cdot \frac{\exp(b\cdot u)-\exp(-b\cdot u)}{2} \Big)^2} \Bigg)\\
\end{align*}
When $u \gg 0$, $\fa$ can be approximated to the following:
\begin{align*}
&	\fa(u; b) \approx \frac{1}{b}\log\Bigg( \exp(b) \cdot \Big(\frac{\exp(b\cdot u)}{2} \Big) +
	\\ 
 & \qquad \sqrt{1 + \Big(\exp(b) \cdot \frac{\exp(b\cdot u)}{2} \Big)^2} \Bigg), \; \exp(-b\cdot u) \to 0 \\
&	 \approx \frac{1}{b}\log\Bigg( \exp(b) \cdot \frac{\exp(b\cdot u)}{2} + \exp(b) \cdot \frac{\exp(b\cdot u)}{2} \Bigg), \\
& \hfill \qquad  \sqrt{1 + u^2} \approx u \quad \text{for} \quad u \gg 0  \\
&	 = \frac{1}{b}\log\Bigg( \exp(b) \cdot \exp(b\cdot u) \Bigg) \\
&	 = \frac{1}{b}\log\Bigg( \exp(b) \Bigg) + \frac{1}{b}\log\Bigg(\exp(b\cdot u) \Bigg) \\
&	 = 1 + u
\end{align*}

Now for $u \ll 0$, we know that $\sinh^{-1}(u)$ and $\sinh(u)$ are odd functions meaning
\begin{align}
\sinh^{-1}(-u) &= -\sinh^{-1}(u) \\
\sinh{(-u)} &= -\sinh(u)
\end{align}
Also, we know that composition of two odd functions is an odd function making $\fa{}$ an odd function.
Now, 
\begin{align*}
	& \fa(u; b) \approx  u + 1 &\qquad \text{for} \qquad u >> 0\\
	\implies	& \fa(u; b) \approx -(-u + 1) &\qquad \text{for} \qquad u << 0 \\
\end{align*}
Hence, to avoid numerical overflow in implementing $\fa{}$, we apply it in piece-wise manner as follows:
\begin{align*}
	\fa(u; b) = \begin{cases*}
		\frac{1}{b}\sinh^{-1}(\exp(b)\sinh(b\cdot u)) \quad  \text{if} \quad |x| \leq 5 \\
		u + 1\cdot \sign(u)	 \qquad \qquad \quad \text{else}
	\end{cases*}
\end{align*}
Similarly, its partial derivative is implemented using:
\begin{align}
	\frac{\partial}{\partial u}\fa(u; b) = \begin{cases*}
		\frac{e^b \cosh(b\cdot u)}{\sqrt{1 + \big(e^b \sinh(b\cdot u)\big)^2}}
		 \quad  \text{if} \quad |x| \leq 5 \\
		1	 \qquad \qquad \qquad \qquad \text{else}
	\end{cases*} \label{eq:fa_imp}
\end{align}

\subsection{Bounds of the derivatives}
\label{sec:bounds}
Assume $\mathbf{D}\fa(u;b) = \frac{\partial }{\partial u} \fa(u;b)$ and $b > 0$.
For larger values of $u$, from eq.~\ref{eq:fa_imp}, $\mathbf{D}\fa(u;b) \approx 1$.
Now, we show the maximum of $\mathbf{D}\fa(u;b)$ for the values $u \in [-5,5]$,
For this we compute $\mathbf{D}^2\fa(u;b)$:
\begin{align*}
	\mathbf{D}^2\fa(u;b) & := \\
	&\frac{be^b \sinh(bu) \left(e^{2b} \sinh^2(bu) - e^{2b}\cosh^2(bu) + 1\right)}{\left(e^{2b} \sinh^2(bu) + 1\right)^{3/2}} \\
				& := \frac{be^b\sinh(bu)(1 - e^{2b})}{\left(e^{2b} \sinh^2(bu) + 1\right)^{3/2}}
\end{align*}
In order to compute the maximum of the function $\mathbf{D}\fa(u;b)$, we equate $\mathbf{D}^2\fa(u;b)$ to zero:
\begin{align*}
	& be^b\sinh(bu)(1 - e^{2b}) = 0  \quad \left({\left(e^{2b} \sinh^2(bu) + 1\right)^{3/2}} > 0\right) \\
	&\implies  \sinh(bu) = 0 \\
	&\implies	 u = 0
\end{align*}
Now, we compute $\mathbf{D}^3\fa(u;b)$ for $u = 0$. $\mathbf{D}^3\fa(u;b)$ can be given as:
\begin{align*}
	\mathbf{D}^3\fa(u;b) = &-\frac{b^2 e^b (2 e^{2b} \sinh^2(bu) - 1) \cosh(bx)}{(e^{2b} \sinh^2(bu) + 1)^{5/2}}\\
 &\cdot (e^{2 b} \sinh^2(bu) - e^{2 b} \cosh^2(bu) + 1)
\end{align*}
Substituting $u = 0$, we get
\begin{align*}
	\mathbf{D}^3\fa(0;b) &= -\frac{b^2 e^b (2 e^{2b} \cdot 0 - 1) \cdot1\cdot (e^{2 b} \cdot0 - e^{2 b} \cdot1 + 1)}{(e^{2b} \cdot0 + 1)^{5/2}} \\
		&	= b^2e^b(1-e^{2b}) \; 
			< 0 \qquad (b > 0)
\end{align*}

Hence, the bounds for the $\mathbf{D}\fa(u;b)$ is $\{1, e^b\}$.

\section{Additional experiments}
\label{sec:add_exp}

\subsection{Experiment on varying the order of the channels}

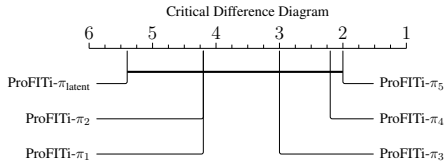
\begin{figure}[h]
	\centering
	\large
		\begin{tikzpicture}[
  treatment line/.style={rounded corners=1.5pt, line cap=round, shorten >=1pt},
  treatment label/.style={font=\normalfont},
  group line/.style={ultra thick}, scale = 0.5
]

\begin{axis}[
  clip={false},
  axis x line={center},
  axis y line={none},
  axis line style={-},
  xmin={1},
  ymax={0},
  scale only axis={true},
  width={\axisdefaultwidth},
  ticklabel style={anchor=south, yshift=1.3*\pgfkeysvalueof{/pgfplots/major tick length}, font=\LARGE},
  every tick/.style={draw=black},
  major tick style={yshift=.5*\pgfkeysvalueof{/pgfplots/major tick length}},
  minor tick style={yshift=.5*\pgfkeysvalueof{/pgfplots/minor tick length}},
  title style={yshift=\baselineskip},
  xmax={6},
  ymin={-4.5},
  height={5\baselineskip},
  xtick={1,2,3,4,5,6},
  minor x tick num={3},
  x dir={reverse},
  title={Critical Difference Diagram},
]

\draw[treatment line] ([yshift=-2pt] axis cs:2.0, 0) |- (axis cs:1.5, -2.0)
  node[treatment label, anchor=west] {ProFITi-$\pi_5$};
\draw[treatment line] ([yshift=-2pt] axis cs:2.2, 0) |- (axis cs:1.5, -4.0)
  node[treatment label, anchor=west] {ProFITi-$\pi_4$};
\draw[treatment line] ([yshift=-2pt] axis cs:3.0, 0) |- (axis cs:1.5, -6.0)
  node[treatment label, anchor=west] {ProFITi-$\pi_3$};
\draw[treatment line] ([yshift=-2pt] axis cs:4.2, 0) |- (axis cs:5.9, -6.0)
  node[treatment label, anchor=east] {ProFITi-$\pi_1$};
\draw[treatment line] ([yshift=-2pt] axis cs:4.2, 0) |- (axis cs:5.9, -4.0)
  node[treatment label, anchor=east] {ProFITi-$\pi_2$};
\draw[treatment line] ([yshift=-2pt] axis cs:5.4, 0) |- (axis cs:5.9, -2.0)
  node[treatment label, anchor=east] {ProFITi-$\pi_\text{latent}$};
\draw[group line] (axis cs:2.0, -1.3333333333333333) -- (axis cs:5.4, -1.3333333333333333);

\end{axis}
\end{tikzpicture}
	\captionof{figure}{Statistical test on the results of various channel orders for ProFITi.}
	\label{fig:cd_perm}
\end{figure}
In ProFITi, we fix the order of channels to make $\fc{}$ equivariant.
In Figure~\ref{fig:cd_perm}, through critical difference diagram, we demonstrate that changing the permutation 
used to fix the channel order does not provide statistically significant difference in the results.
ProFITi$-\pi_{1:5}$ indicate ProFITi with $5$ different pre-fixed permutations on channels
while time points are left in causal order.
The order in which we sort channels and time points is a hyperparameter.
To avoid this hyperparamerter and even allow different sorting criteria
for different instances, one can parametrize $P_\pi$ as a function
of $X$ (\textbf{learned sorted triangular invertible attention}).
ProFITi$-\pi_\text{latent}$ indicate ProFITi where the permutation of all the observations
(including channels and time points) are set on the latent embedding.
Specifically, we pass $x$ through an MLP and selected the permutation by sorting its output.
Significant difference in results is not observed because
the ordering in lower triangular matrix can be seen as a Bayesian network,
and the graph with the triangular matrix as adjacency is a full directed graph,
and all of them induce the same factorization.

\subsection{Varying observation and forecast horizons}
\label{sec:scalability}

In Table~\ref{tab:scale}, we compare ProFITi with 
two next best models, GraFITi+ and Neural Flows.
Our evaluation involves varying the observation and forecast
horizons on the Physionet'12 dataset. Furthermore, we also
compare with ProFITi-$A^\text{tri}$+$A^\text{reg}$, 
wherein the triangular attention mechanism in ProFITi
is replaced with a regularized attention mechanism.

ProFITi exhibits superior performance compared to both Neural Flows and GraFITi+,
demonstrating a significant advantage. 
We notice that when we substitute $A^\text{tri}$ with $A^\text{reg}$; 
this change leads to a degradation in performance as the forecast sequence length increases.
Also, note that the run time for computing $A^\text{reg}$ and its determinant is an order of magnitude larger
than that of $A^\text{tri}$. 
This is because, it requires $\bigO(K^3)$
complexity to compute spectral radius $\sigma(A)$ and determinant of $A^\text{reg}$,
whereas computing determinant of $A^\text{tri}$ requires $\bigO(K)$ complexity.

Additionally, we see that as the sequence length increases,
there is a corresponding increase in the variance of the njNLL.
This phenomenon can be attributed to the escalating number of target values ($K$),
which increases with longer sequences.
Predicting the joint distribution over a larger set of target values
can introduce noise into the results, thereby amplifying the variance in the outcomes.
Whereas for the GraFIT+ and Neural Flows it is not the case as they predict only marginal distributions.
Further, as expected the njNLL of all the models decrease with increase in sequence lengths as
it is difficult to learn longer horizons compared to short horizons of the forecast.

In Figure~\ref{fig:scalability}, we show the qualitative performance of ProFITi. 
We compare the trajectories predicted by ProFITi by random sampling of $z$ with the
distribution predicted by the GraFITi+ (next best model).
\begin{table*}
	\centering
	\fontsize{8}{10}\selectfont
	\caption{Varying observation and forecast horizons of Physionet'12 dataset}
	\label{tab:scale}
	\begin{tabular}{lcrrcrrcrr}
		\toprule
					& \multicolumn{3}{c}{obs/forc : 36/12hrs}	& \multicolumn{3}{c}{obs/forc : 24/24hrs}	& \multicolumn{3}{c}{obs/forc : 12/36hrs}	\\
					&	njNLL	& \multicolumn{2}{c}{run time (s)} 		&	njNLL	& \multicolumn{2}{c}{run time (s)} 		&	njNLL	& \multicolumn{2}{c}{run time (s)} 			\\
					&	& epoch &	$A$ &	& epoch &	$A$ &	& epoch &	$A$ 	\\
					\hline
		Neural Flows	&	0.709$\pm$0.483	& 109.6	&	-	&	1.097$\pm$0.044	&	46.6	&	-	&	1.436$\pm$0.187	&	45.5	&	-	\\
		GraFITi+	&	0.522$\pm$0.015	&	42.9&-	&	0.594$\pm$0.009	&	43.1&	-	&	0.723$\pm$0.004	&	37.5&	-	\\
		ProFITi		&	-0.768$\pm$0.041 &	64.8	&	3.3& -0.355$\pm$0.243	&	66.2	&	5.2	&	-0.291$\pm$0.415	&	82.1&	8.6	\\
		ProFITi-$A^\text{tri}$+$A^\text{reg}$	&	-0.196$\pm$0.096	&	89.9&	7.1 	&	0.085$\pm$0.209	&	142.1	&	30.1	&	0.092$\pm$0.168	&	245.8&	73.1	\\
		\bottomrule
	\end{tabular}
\end{table*}

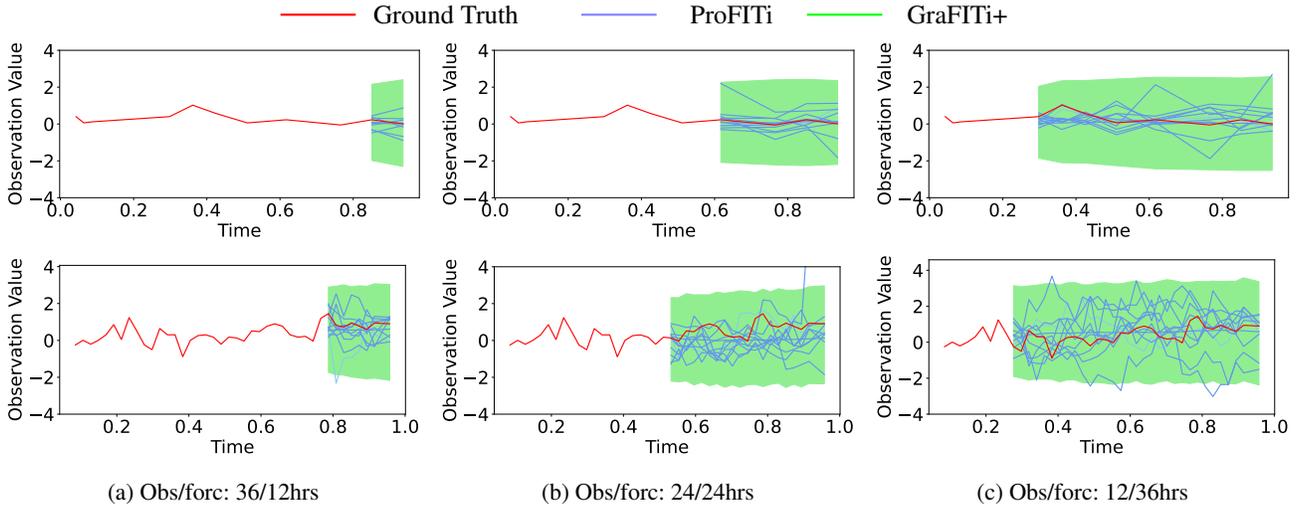
\begin{figure*}
	\centering
	
	\begin{tikzpicture}
		\draw[red, thick] (0,0) -- (1,0);
		\node at (2.2,0) {Ground Truth};
		
		\draw[blue!50!, thick] (4,0) -- (5,0);
		\node at (6,0) {ProFITi};
		
		\draw[green, thick] (7,0) -- (8,0);
		\node at (9,0) {GraFITi+};
	\end{tikzpicture}

	\captionsetup[subfigure]{justification=centering}
	\begin{subfigure}{0.32\textwidth}
		\includegraphics[width=\textwidth]{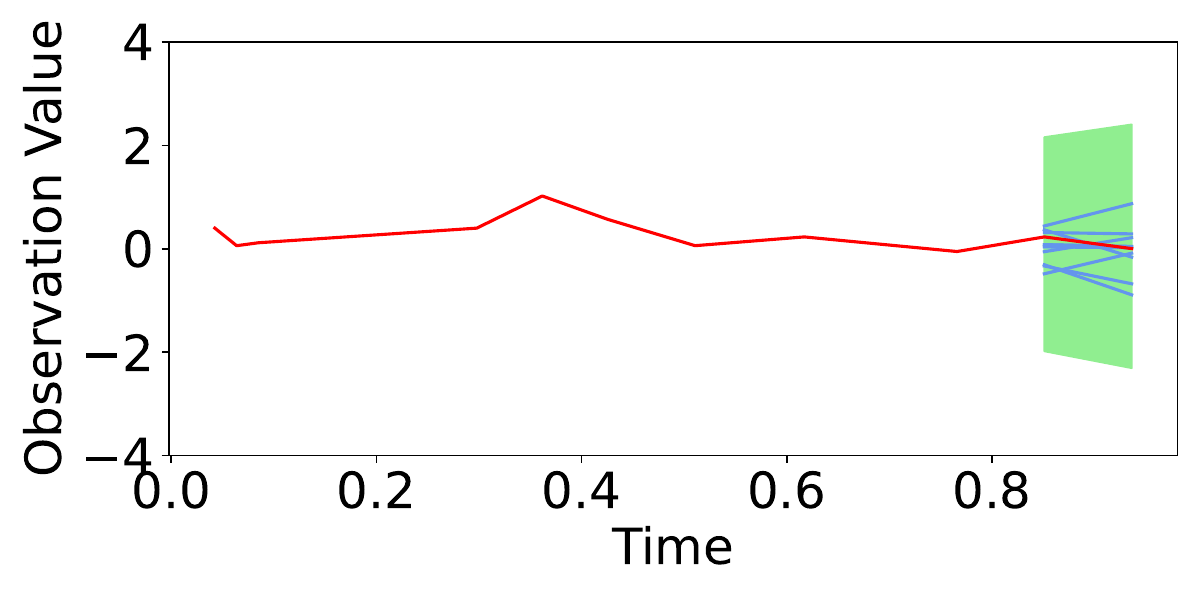}
	\end{subfigure}
	\begin{subfigure}{0.32\textwidth}
		\includegraphics[width=\textwidth]{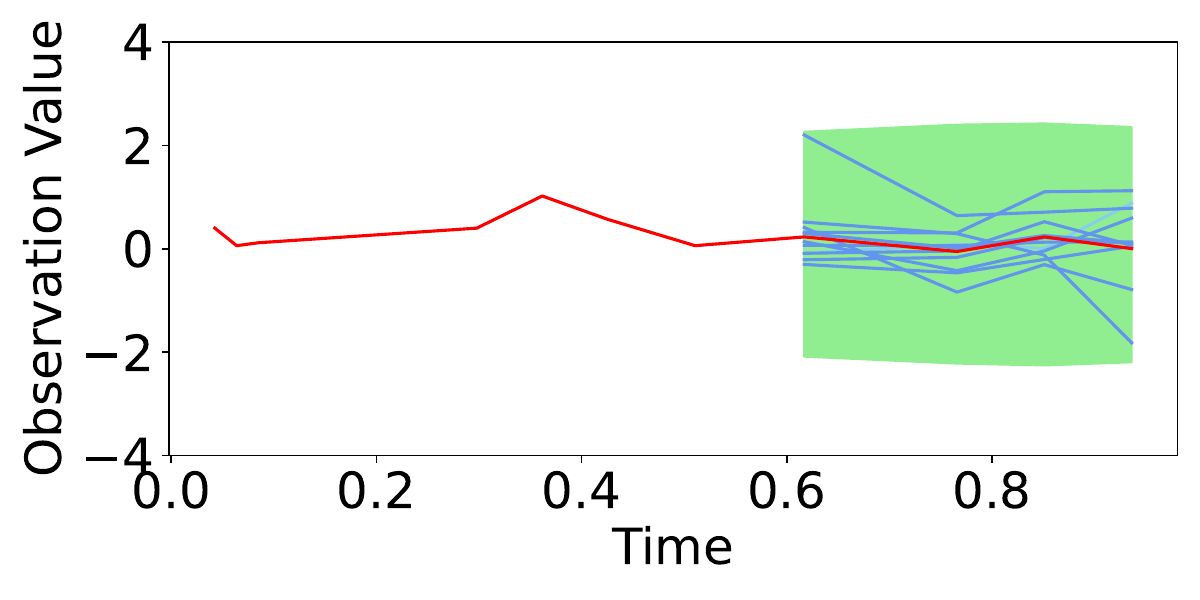}
	\end{subfigure}
	\begin{subfigure}{0.32\textwidth}
		\includegraphics[width=\textwidth]{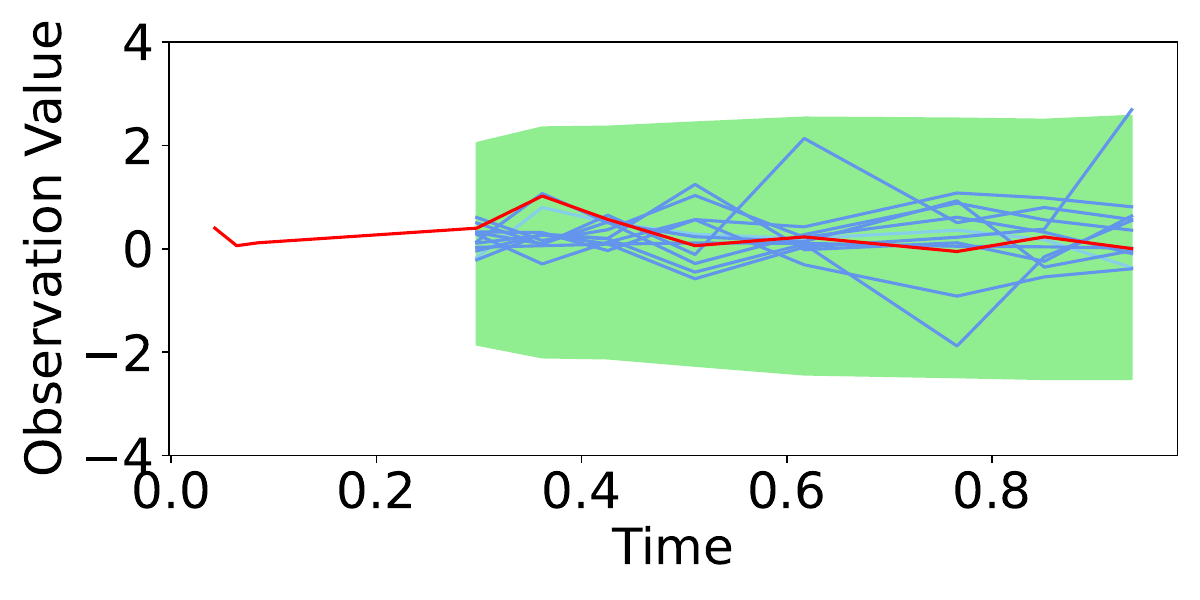}
	\end{subfigure}

	\begin{subfigure}{0.32\textwidth}
		\includegraphics[width=\textwidth]{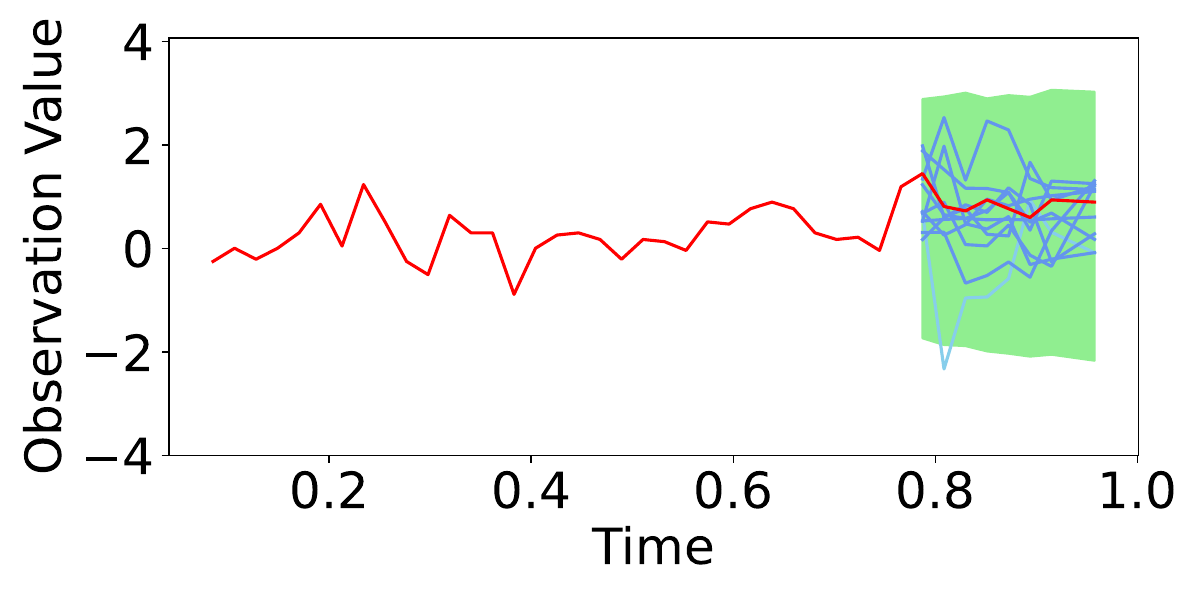}
		\subcaption{Obs/forc: 36/12hrs}
	\end{subfigure}
	\begin{subfigure}{0.32\textwidth}
		\includegraphics[width=\textwidth]{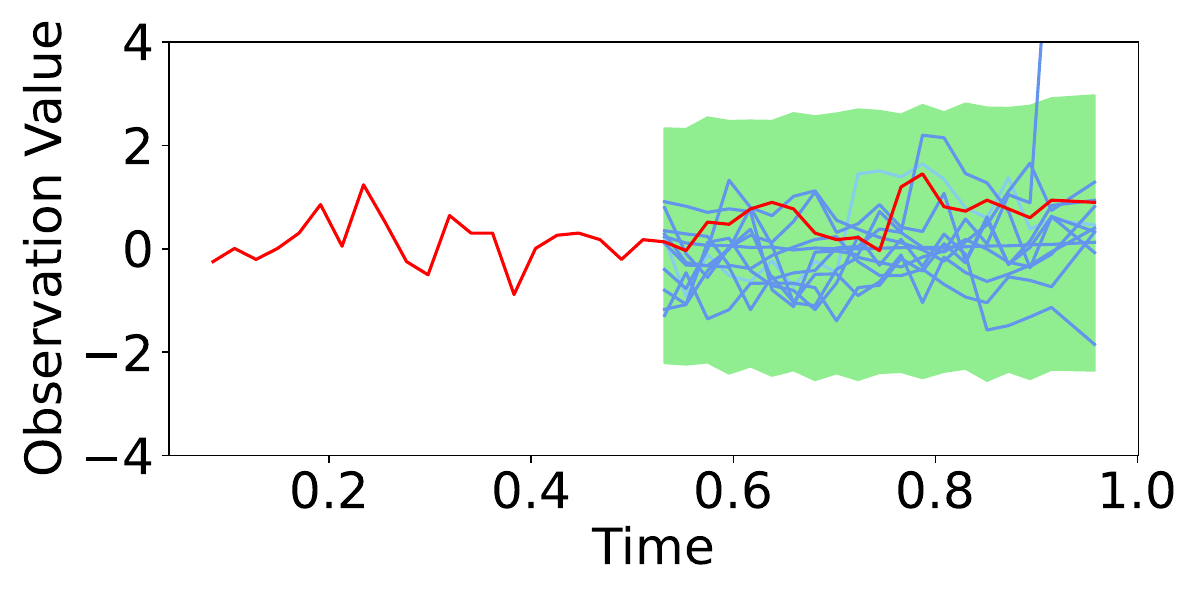}
		\subcaption{Obs/forc: 24/24hrs}
	\end{subfigure}
	\begin{subfigure}{0.32\textwidth}
		\includegraphics[width=\textwidth]{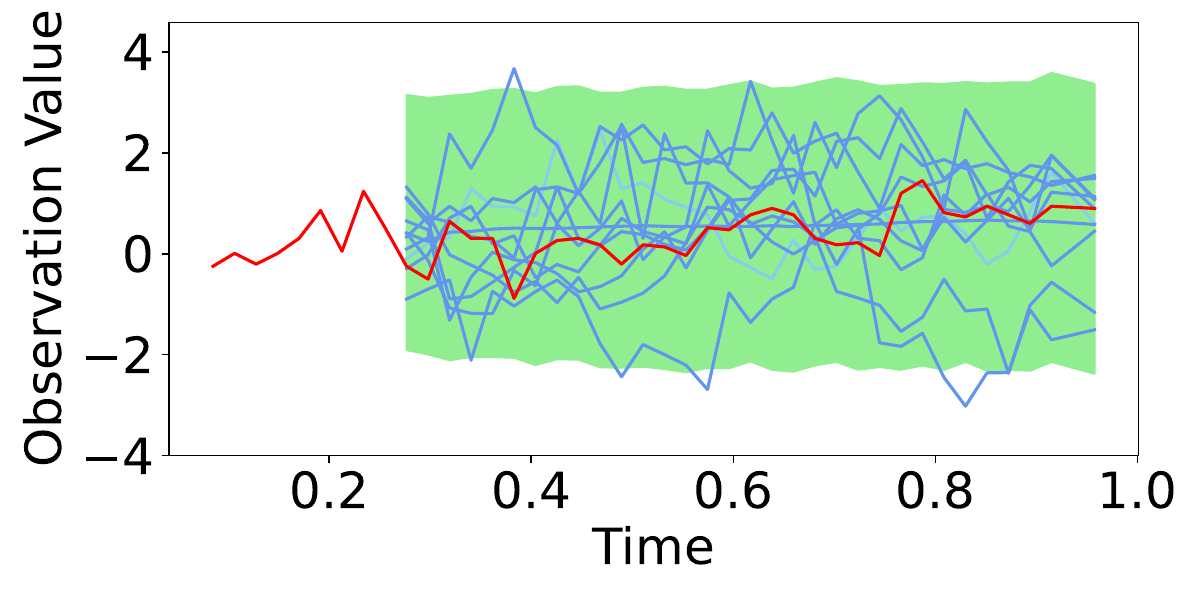}
		\subcaption{Obs/forc: 12/36hrs}
	\end{subfigure}
	\caption{Demonstrating (10) trajectories generated using ProFITi for Physionet'12 dataset.}
	\label{fig:scalability}
\end{figure*}

\subsection{Experiment with varying number of missing values}

\begin{table}[h]
	\centering
	\small
	\caption{Varying $\#$observations in the time series. Physionet'12 dataset, evaluation metric njNLL.}
	\begin{tabular}{lccc}
		\hline
		&\multicolumn{3}{c}{$\%$ missing observations} \\
		\hline
		& $10\%$	&	$50\%$	&	$90\%$	\\
		\hline
	Neural Flow	&	0.497$\pm$0.042	&	0.542$\pm$0.031&	0.677$\pm$0.018	\\
	GraFITi+	&	0.402$\pm$0.016 &	0.481$\pm$0.018	&	0.666$\pm$0.012	\\
	ProFITi		&	\textbf{-0.141$\pm$0.036}	&	\textbf{0.077$\pm$0.012	}&	\textbf{0.336$\pm$0.033}\\		
	\hline
	\end{tabular}
\end{table}

Here, we experimented on Physionet'12 dataset with varying sparsity levels. We randomly removed $x\%$, $x \in \{10, 50, 90\}$ of observations in the series. Compared the performance with GraFITi+ and Neural Flow. We observe that even with $90\%$ missing values, ProFITi perform significantly better.

\subsection{Experiment with varying time sparsity}
\begin{table}[H]
	\centering
	\small
	\caption{Varying $\#$observation events i.e., time points in the time series. Physionet'12 dataset, evaluation metric njNLL.}
	\begin{tabular}{lccc}
		\hline
		&\multicolumn{3}{c}{$\%$ missing observation events} \\
		\hline
		& $10\%$	&	$50\%$	&	$90\%$	\\
		\hline
		\hline
		Neural Flow	&	0.528$\pm$0.037	&		0.578$\pm$0.048&	0.858$\pm$0.006	\\
		GraFITi+	&	0.469$\pm$0.032&	0.520$\pm$0.022	&	0.767$\pm$0.004\\
		ProFITi		&	\textbf{-0.106$\pm$0.112}&	\textbf{-0.160$\pm$0.056}	&	\textbf{0.128$\pm$0.056}\\		
		\hline
	\end{tabular}
\end{table}

\begin{table*}
	\small
	\centering
	\caption{Comparing mNLL of ProFITi and ProFITi\_marg to verify marginal consistency.}
	\label{tab:marg_consist}
	\begin{tabular}{lcccc}
		\toprule
		&	\multicolumn{1}{c}{USHCN}	&	\multicolumn{1}{c}{Physionet,12}	&	\multicolumn{1}{c}{MIMIC-III}	&	\multicolumn{1}{c}{MIMIC-IV}\\
		\midrule
		ProFITi (ours) & 	-3.319$\pm$0.229	&	-0.095$\pm$0.057		&	0.477$\pm$0.023	&	-0.189$\pm$0.024 \\
		ProFITi\_marg (ours)	&	-2.575$\pm$1.336	&	-0.368$\pm$0.033	&	0.092$\pm$0.036	&	-0.782$\pm$0.023	\\
		\bottomrule
	\end{tabular}
\end{table*}

We use Physionet'12 dataset to experiment on varying number of observation events i.e. time points. We randomly removed $x\%, x \in \{10, 50, 90\}$ of observation events in the series and compared GraFITi+, Neural Flow and ProFITi. Again, we observe that even with $90\%$ of time points missing, ProFITi perform significantly better.

\section{Marginal consistency}
\label{sec:marg_consistency}

We notice that ProFITi has a limitation in terms of Marginal consistency.
A model should be \textbf{consistent w.r.t. margins},
i.e., yield the marginal distribution for any subquery
$\{k_1,\ldots,k_S\}\subset \{1,\ldots,K\}$:
	\begin{align*}
		& \hat p(y_{k_1},\ldots,y_{k_S} \mid x\obs,x\qu_{k_1},\ldots,x\qu_{k_S})
		\\ & = \int\limits_{y_{\bar k_1},\ldots,y_{\bar k_{K-S}}}
		\hat p(y_1,\ldots,y_K \mid x\obs,x\qu_1,\ldots,x\qu_K)
		dy_{\bar k_1}\cdots dy_{\bar k_{K-S}}, 
		\\ & \quad \text{with }
		\{\bar k_1,\ldots,\bar k_{K-S}\}:=
		\{1,\ldots,K\} \setminus \{ k_1,\ldots,k_S\}
	\end{align*}
ProFITi is not guaranteed to have this property, hence, it provide 
inconsistent results on marginals. To demonstrate this,
we trained the models ProFITi and ProFITi\_marg
for njNLL and compute the mNLL. 
Results are presented in Table~\ref{tab:marg_consist}.
For ProFITi at the time of inference
we zeroed the off diagonal elements of $A^\text{tri}$ in SITA to get
mNLL.

Other than USHCN, ProFITi\_marg yield results that are vastly better
than ProFITi. This is because of marginal inconsistency of the ProFITi model.
On the other hand, for USHCN dataset, ProFITi has better mNLL than ProFITi\_marg.
However, ProFITi\_marg has as large standard deviation because one fold had poor
mNLL whereas other folds have mNLL comparable to that of ProFITi.

\section{Hyperparameters searched}
Following the original works of the baseline models, we search the following hyperparameters:

 \paragraph{HETVAE~\citep{Shukla2022.Heteroscedastic}:}
 \begin{itemize}
 	\item Latent Dimension: \{8, 16, 32, 64, 128\}
 	\item Width	:	\{128,256,512\}
 	\item \# Reference Points:	\{4, 8, 16, 32\}
 	\item \# Encoder Heads:	\{1, 2, 4\}
 	\item MSE Weight:	\{1, 5, 10\}
 	\item Time Embed. Size:	\{16, 32, 64, 128\}
 	\item Reconstruction Hidden Size: \{16, 32, 64, 128\}
 \end{itemize}

\paragraph{GRU-ODE-Bayes~\citep{DeBrouwer2019.GRUODEBayes}:}
\begin{itemize}
	\item solver:	\{euler, dopri5\}
	\item \# Hidden Layers:	\{3\}
	\item Hidden Dim.:	\{64\}
\end{itemize}

\paragraph{Neural Flows~\citep{Bilos2021.Neural}:}
\begin{itemize}
	\item Flow Layers:	\{1, 4\}
	\item \# Hidden Layers:	\{2\}
	\item Hidden Dim.:	\{64\}
\end{itemize}

\paragraph{CRU~\citep{Schirmer2022.Modeling}:}
\begin{itemize}
	\item \# Basis:	\{10, 20\}
	\item Bandwidth:	\{3, 10\}
	\item lsd:	\{10, 20, 30\}
\end{itemize}

\paragraph{CNF+:}
\begin{itemize}
	\item \# Attention layers: \{1,2,3,4\}
	\item \# Projection matrix dimension for attention: \{32,64,128,256\}
\end{itemize}

\paragraph{GraFITi+~\citep{Yalavarthi2024.GraFITi}:}
\begin{itemize}
	\item \# layers:	\{2, 3, 4\}
	\item \# MAB heads:	\{1, 2, 4\}
	\item Latent Dim.:	\{32, 64, 128\}
\end{itemize}

\paragraph{ProFITi (Ours):}
\begin{itemize}
	\item \# Flow layers:	\{7, 8, 9, 10\}
	\item $\epsilon$:	\{1e-5\}
	\item Latent Dim.:	\{32, 64, 128, 256\}
\end{itemize}
\end{document}